\pdfoutput=1

\documentclass{article}

\PassOptionsToPackage{numbers,sort&compress}{natbib}

\usepackage[final]{neurips_2021}
\usepackage{mathrsfs}  

\usepackage[utf8]{inputenc}
\usepackage[T1]{fontenc}   
\usepackage[dvipsnames]{xcolor}
\usepackage[colorlinks,linkcolor=black,citecolor=blue,urlcolor=blue,bookmarks=true]{hyperref}
\usepackage{amsmath,amsfonts,bm,amssymb,mathtools,amsthm}
\usepackage{cleveref}
\usepackage{url}           
\usepackage{booktabs}      
\usepackage{nicefrac}      
\usepackage{microtype}     
\usepackage{physics}
\usepackage{caption}
\usepackage{tabu}
\usepackage{subcaption}
\usepackage{adjustbox}
\usepackage{pifont}
\usepackage{tikz}
\usepackage{graphics}
\usepackage{pgfplotstable}
\usepackage{fontawesome}
\newcommand{\cmark}{\ding{51}}%
\newcommand{\xmark}{\ding{55}}%

\def\1{\bm{1}}

\def\tL{{\mathsf{L}}}

\def\tP{{\mathsf{P}}}

\def\rvw{{\mathbf{w}}}

\def\rvz{{\mathbf{z}}}

\def\ervw{{\textnormal{w}}}

\def\gB{{\mathcal{B}}}

\def\gO{{\mathcal{O}}}

\newcommand{\R}{\mathbb{R}}

\DeclareMathOperator*{\argmin}{arg\,min}

\DeclareMathOperator{\sign}{sign}

\usepackage{amsthm}
\usepackage{thmtools}

\newcommand{\xiv}{\boldsymbol{\xi}}
\newcommand{\EE}{\mathbf{E}}
\newcommand{\rsf}{\mathsf{r}}
\newcommand{\ssf}{\mathsf{s}}
\newcommand{\ws}{\mathsf{w}}
\newcommand{\wv}{\rvw}
\newcommand{\zv}{\rvz}

\newcommand{\env}{\mathscr{M}}
\newcommand{\inner}[2]{\left\langle#1,#2\right\rangle}
\newcommand{\peta}{\mu}

\newcommand{\id}{\mathsf{id}}
\newcommand{\LP}[2][\varrho]{\tL^{#1}_{#2}}
\theoremstyle{definition}

\newtheorem{theorem}{Theorem}[section]
\newtheorem*{theorem*}{Theorem}

\newtheorem{lemma}[theorem]{Lemma}

\newtheorem{example}[theorem]{Example}
\newtheorem{prop}[theorem]{Proposition}
\usepackage{float}
\usepackage{enumitem}
\usepackage{wasysym}
\usepackage{wrapfig}
\usepackage{siunitx}
\usepackage{knitting}
\usepackage{multirow}
\usepackage{tikz}
\usepackage{pgfplots}
\usepgfplotslibrary{colorbrewer}
\usetikzlibrary{pgfplots.colorbrewer}
\pgfplotsset{cycle list/Set1}

\usetikzlibrary{positioning}
\usetikzlibrary{arrows}
\usetikzlibrary{calc}

\tikzset{state/.style={draw, rectangle, thick}}
\pgfplotsset{compat=1.13}
\usepgfplotslibrary{patchplots}
\usepackage{lipsum}
\usepackage{stmaryrd, graphicx}
\makeatletter
\newcommand{\fixed@sra}{$\vrule height 2\fontdimen22\textfont2 width 0pt\shortrightarrow$}
\newcommand{\shortarrow}[1]{%
  \mathrel{\text{\rotatebox[origin=c]{\numexpr#1*45}{\fixed@sra}}}
}

\newif\ifdraft
\draftfalse

\declaretheoremstyle[%
  spaceabove=-6pt,%
  spacebelow=6pt,%
  headfont=\normalfont\itshape,%
  postheadspace=1em,%
  qed=\qedsymbol%
]{mystyle}

\title{Demystifying and Generalizing BinaryConnect}

\author{Tim Dockhorn \thanks{Work done during an internship at Huawei Noah's Ark Lab. Correspondence to \texttt{tim.dockhorn@uwaterloo.ca}.}\\ 
        University of Waterloo \\
        \And 
        Yaoliang Yu \\
        University of Waterloo \\
        \And 
        Eyyüb Sari \\
        Huawei Noah’s Ark Lab \\
        \And 
        Mahdi Zolnouri \\
        Huawei Noah’s Ark Lab \\
        \And 
        Vahid Partovi Nia \\
        Huawei Noah’s Ark Lab \\
}
\begin{document}

\maketitle

\begin{abstract}
    BinaryConnect (BC) and its many variations have become the de facto standard for neural network quantization. However, our understanding of the inner workings of BC is still quite limited. We attempt to close this gap in four different aspects: (a) we show that existing quantization algorithms, including post-training quantization, are surprisingly similar to each other; (b) we argue for proximal maps as a natural family of quantizers that is both easy to design and analyze; (c) we refine the observation that BC is a special case of dual averaging, which itself is a special case of the generalized conditional gradient algorithm; (d) consequently, we propose \emph{ProxConnect} (PC) as a generalization of BC and we prove its convergence properties by exploiting the established connections. We conduct experiments on CIFAR-10 and ImageNet, and verify that PC achieves competitive performance.
\end{abstract}
\section{Introduction} \label{sec:introduction}
Scaling up to extremely large datasets and models has been a main ingredient for the success of deep learning. 
Indeed, with the availability of big data, more computing power, convenient software, and a bag of training tricks as well as algorithmic innovations, 
the size of models that we routinely train in order to achieve state-of-the-art performance has exploded, e.g., to 
billions of parameters in recent language models~\citep{gpt3}. However, high memory usage and computational cost at inference time has made it difficult to deploy these models in real-time or on resource-limited devices~\citep{lin2020mcunet}. The environmental impact of training and deploying these large models has also been recognized~\citep{strubell2019energy}. A common approach to tackle these problems is to compress a large model through quantization, i.e., replacing high-precision parameters with lower-precision ones. For example, we may constrain a subset of the weights to be binary~\citep{courbariaux2015,rastegari2016,yin2018,bai2018,nia2018binary,Martinez2020Training,ajanthan2021} or ternary~\citep{li2016arxiv, zhu2016}. Quantization can drastically decrease the carbon footprint of training and inference of neural networks, however, it may come at the cost of increased bias~\citep{hooker2020characterising}.

One of the main methods to obtain quantized neural networks is to encourage quantized parameters during gradient training using explicit or implicit regularization techniques, however, other methods are possible~\citep{GuptaAGN15,han2016,ZhouYGXC17,ParkAY17,LengDLZJ18,HelwegenWGLCN19,han20d}. Besides the memory benefits, the structure of the quantization can speed up inference using, for example, faster matrix-vector products~\citep{han2016, hubara2017}. Training and inference can be made even more efficient by also quantizing the activations~\citep{rastegari2016} or gradients \citep{ZhouWNZWZ16}. Impressive performance has been achieved with quantized networks, for example, on object detection~\citep{yin2019} and natural language processing~\citep{xu2018} tasks. The theoretical underpinnings of quantized neural networks, such as when and why their performance remains reasonably  well, have been actively studied~\citep{anderson2018,DingLXS19,HouZK19,WangZZTWLYWL19,YinLZOQX19}.

BinaryConnect~\citep[BC,][]{courbariaux2015} and its many variations~\citep[][]{rastegari2016, zhu2016,ChenWP19} are considered the gold standard for neural network quantization. Compared to plain (stochastic) gradient descent, BC does not evaluate the gradient at the current iterate but rather at a (close-by) quantization point using the Straight Through Estimator~\citep{Bengio_STE_2013}. Despite its empirical success, BC has largely remained a ``training trick''~\citep{ajanthan2021} and a rigorous understanding of its inner workings has yet to be found, with some preliminary steps taken in \citet{LiDXSSG17} for the convex setting and in~\citet{yin2018} for a particular quantization set. As pointed out in~\citet{bai2018},  
BC only evaluates gradients at the finite set of quantization points, and therefore does not exploit the rich information carried in the continuous weights network.
\citet{bai2018} also observed that BC is formally equivalent to the dual averaging algorithm~\citep{nesterov2009, xiao2010}, while some similarity to the mirror descent algorithm was found in~\citet{ajanthan2021}.

The main goal of this work is to significantly improve our understanding of BC, by connecting it with well-established theory and algorithms. In doing so we not only simplify and improve existing results but also obtain novel generalizations. 
We summarize our main contributions in more details:
\begin{itemize}[leftmargin=*]
    \item In \Cref{sec:background}, we show that existing gradient-based quantization algorithms are surprisingly similar to each other: the only high-level difference is at what points we evaluate the gradient and perform the update. 
    \item In \Cref{sec:prox}, we present a  principled theory for constructing proximal quantizers. Our results unify previous efforts, remove tedious calculations, and bring theoretical convenience. We illustrate our theory by effortlessly designing a new quantizer that  
    can be computed in one-pass, works for different quantization sets (binary or multi-bit), and includes previous attempts as special cases~\citep{courbariaux2015,yin2018,bai2018,zhu2016}.
    \item In \Cref{sec:method}, we significantly extend the observation of \citet{bai2018} that the updates of BC are the same as the dual averaging algorithm~\citep{nesterov2009,xiao2010}: BC is a nonconvex counterpart of dual averaging, and more importantly, dual averaging itself is simply the generalized conditional gradient algorithm applied to a smoothened dual problem. The latter fact, even in the convex case, does not appear to be widely recognized to the best of our knowledge.
    \item In \Cref{sec:PC}, making use of the above established results, we propose~\emph{ProxConnect}~(PC) as a family of algorithms that generalizes BC and we prove its convergence properties for both the convex and the nonconvex setting. We rigorously justify the diverging parameter in proximal quantizers and resolve a discrepancy between theory and practice in the literature~\citep{bai2018,yin2018,ajanthan2021}.
    \item In \Cref{sec:experiments}, we verify that PC outperforms BC and ProxQuant~\citep{bai2018} on CIFAR-10 for both fine-tuning pretrained models as well as end-to-end training. On the more challenging ImageNet dataset, PC yields competitive performance despite of minimal
    hyperparameter tuning.
\end{itemize}
\section{Background} \label{sec:background}
We consider the usual (expected) objective $\ell(\rvw) = \EE\ell(\rvw, \xiv)$, where $\xiv$ represents random sampling. For instance, $\ell(\rvw) = \frac{1}{n} \sum_{i=1}^n \ell_i(\rvw)$ where $\xiv$ is uniform over $n$ training samples (or minibatches thereof), $\wv$ are the weights of a neural network and $\ell_i$ may be the cross-entropy loss (of the $i$-th training sample). We denote a sample (sub)gradient of $\ell$ at $\rvw$ and $\xiv$ as $\widetilde \nabla \ell(\rvw) = \nabla \ell(\rvw, \xiv)$ so that $\EE \widetilde\nabla \ell(\wv) = \nabla \ell(\wv)$. Throughout, we use the starred notation $\wv^*$ for continuous weights and reserve~$\wv$ for (semi)discrete ones.

We are interested in solving the following (nonconvex) problem: 
\begin{align}
\label{eq:prob}
\min_{\rvw \in Q} ~ \ell(\rvw), 
\end{align}
where $Q \subseteq \R^d$ is a discrete, nonconvex quantization set. For instance, on certain low-resource devices it may be useful or even mandatory to employ binary weights, i.e., $Q = \{\pm 1\}^d$. Importantly, our goal is to compete against the non-quantized, continuous weights network (i.e. $Q=\R^d$). In other words, we do \emph{not} necessarily aim to solve (the hard, combinatorial) problem \eqref{eq:prob} globally and optimally. Instead, we want to find discrete weights $\wv\in Q$ that remain satisfactory when compared to the non-quantized continuous weights. This is how we circumvent the difficulty in \eqref{eq:prob} and more importantly how the structure of $\ell$ could come into aid. If $\ell$ is reasonably smooth, a close-by quantized weight of a locally, or even globally, optimal continuous weight will likely yield similar performance. 

Tremendous progress has been made on quantizing and compressing neural networks. While it is not possible to discuss all related work, below we recall a few families of gradient-based quantization algorithms that directly motivate our work; more methods can be found in recent surveys~\citep{qin2020binary, guo2018survey}.

\textbf{BinaryConnect (BC).} \citet{courbariaux2015} considered binary networks where $Q = \{\pm1\}^d$ and proposed the BinaryConnect algorithm:
\begin{align}
\wv_{t+1}^* &= \wv_t^* - \eta_t \widetilde\nabla \ell(\tP(\wv_t^*)), 
\end{align}
where $\tP$ is a projector that quantizes the continuous weights $\wv_t^*$ either deterministically (by taking the sign) or stochastically. Note that the (sample) gradient is evaluated at the quantized weights $\wv_t \coloneqq \tP(\wv_t^*)$, while its continuous output, after scaled by the step size $\eta_t$, is added to the continuous weights $\wv_t^*$. 
For later comparison, it is useful to break the BC update into the following two pieces:
\begin{align}
\label{eq:BC}
\wv_{t} = \tP(\wv_{t}^*), 
\qquad
\wv_{t+1}^* = \wv_t^* - \eta_t \widetilde\nabla \ell(\wv_t)
.
\end{align}
Other choices of $\tP$~\citep[][]{ajanthan2021, yin2018} and $Q$~\citep[][]{yin2019, zhu2016} in this framework have also been experimented with.

\looseness=-1

\textbf{ProxQuant (PQ).} \citet{bai2018} applied the usual proximal gradient to solve \eqref{eq:prob}:
\begin{align}
\wv_{t+1} &= \tP(\wv_t - \eta_t \widetilde\nabla \ell(\wv_t)),
\end{align}
which can be similarly decomposed into:
\begin{align}
\wv_{t} = \tP(\wv_{t}^*), 
\qquad
\wv_{t+1}^* = \wv_t - \eta_t \widetilde\nabla \ell(\wv_t)
.
\end{align}
Thus, the only high-level difference between BC and PQ is that the former updates the continuous weights $\wv_t^*$ while the latter updates the quantized weights $\wv_t$. This seemingly minor difference turns out to cause drastically different behaviors of the two algorithms. For example, choosing $\tP$ to be the Euclidean projection to $Q$ works well for BinaryConnect but not at all for ProxQuant.

\textbf{Reversing BinaryConnect.} The above comparison naturally suggests a reversed variant of BC:
\begin{align}
\wv_{t} = \tP(\wv_{t}^*), 
\qquad
\wv_{t+1}^* = \wv_t - \eta_t \widetilde\nabla \ell(\wv_t^*), 
\end{align}
which amounts to switching the continuous and quantized weights in the BC update. Similar to PQ, the choice of $\tP$ is critical in this setup. To the best of our knowledge, this variant has not been formally studied before. Reversing BinaryConnect may not be intuitive, however, it serves as a starting point for a more general method (see~\Cref{sec:PC}).

\textbf{Post-Training Quantization.} Lastly, we can also rewrite the naive post-training quantization scheme in a similar form:
\begin{align}
\wv_{t} = \tP(\wv_{t}^*), 
\qquad
\wv_{t+1}^* &= \wv_t^* - \eta_t \widetilde\nabla \ell(\wv_t^*),
\end{align}
where we simply train the continuous network as usual and then quantize at the end. Note that the quantized weights $\wv_t$ do not affect the update of the continuous weights $\wv_t^*$.
\section{What Makes a Good Quantizer?}
\label{sec:prox}
As we have seen in~\Cref{sec:background}, the choice of the quantizer $\tP$ turns out to be a crucial element for solving \eqref{eq:prob}. Indeed, if $\tP = \tP_Q$ is the projector onto the discrete quantization set $Q$, then BC (and ProxQuant) only evaluate the gradient of $\ell$ at (the finite set of) points in $Q$. As a result, the methods will not be able to exploit the rich information carried in the continuous weights network, which can lead to non-convergence~\citep[Fig. 1b,][]{bai2018}. Since then, many semi-discrete quantizers, that turn continuous weights into more and more discrete ones, have been proposed~\citep[][]{ajanthan2021, yin2018, bai2018,nia2018binary}. 
In this section, we present a principled way to construct quantizers that unifies previous efforts, removes tedious calculations, and also brings theoretical convenience when it comes down to analyzing the convergence of the algorithms in \Cref{sec:background}. 

Our construction is based on the proximal map $\tP_{\rsf}^\peta \colon \R^d \rightrightarrows \R^d$ of a (closed) function $\rsf$:
\begin{align} \label{eq:proximal_operator}
    \tP_{\rsf}^\peta (\rvw^*) = \argmin_{\rvw} ~ \tfrac{1}{2\peta} \|\rvw - \rvw^*\|_2^2 + \rsf(\rvw),
\end{align}
where $\peta >0$ is a smoothing parameter. 
The proximal map is well-defined as long as the function $\rsf$ is lower bounded by a quadratic function, in particular, when $\rsf$ is bounded from below. 
If $\rsf$ is proper and (closed) convex, then the minimizer on the right-hand side of~\eqref{eq:proximal_operator} is uniquely attained, while for general (nonconvex) functions the proximal map may be multi-valued (hence the notation $\rightrightarrows$). If $\rsf = \iota_Q$ is the indicator function (see \eqref{eq:ind} below), then $\tP_{\rsf}^\peta$ reduces to the familiar projector $\tP_Q$ (for any $\peta$). Remarkably, 
a complete characterization of such maps on the real line is available: 
\begin{theorem}[{\cite[Proposition 3,][]{yu2015}}]
A (possibly multi-valued) map $\tP \colon \R \rightrightarrows \R$ is a proximal map (of some function $\rsf$) iff it is (nonempty) compact-valued, monotone and has a closed graph. The underlying function $\rsf$ is unique (up to addition of constants) iff $\tP$ is convex-valued, while $\rsf$ is convex iff $\tP$ is nonexpansive (i.e. $1$-Lipschitz continuous).
\label{thm:prox}
\end{theorem}
The sufficient and necessary conditions of~\Cref{thm:prox} allow one to design proximal maps $\tP$ directly, without needing to know the underlying function $\rsf$ at all (even though it is possible to integrate a version of $\rsf$ from $\tP$). The point is that, as far as quantization algorithms are concerned, having $\tP$ is enough, and hence one is excused from the tedious calculations in deriving $\tP$ from $\rsf$ as is typical in existing works. 

For example, the (univariate) mirror maps constructed in~\citet[Theorem 1,][]{ajanthan2021}, such as $\tanh$, are proximal maps according to \Cref{thm:prox}: In our notation, the (Mirror Descent) updates of~\citet{ajanthan2021} are the same as~\eqref{eq:BC} with $\tP = (\nabla \Phi)^{-1}$ for some mirror map $\Phi$. Since $\Phi$ is taken to be strictly convex in~\citet{ajanthan2021}, it is easy to verify that $\tP$ satisfies all conditions of~\Cref{thm:prox}.\footnote{The above reasoning hinges on $\Phi$ being univariate. More generally, if a multivariate mirror map $\Phi$ is 1-strongly convex (as is typical in Mirror Descent), then it follows from \citet[Corollaire 10.c,][]{moreau1965proximite} that $\tP = (\nabla \Phi)^{-1}$, being a nonexpansion, is again a proximal map (of some convex function).}

The next result allows us to employ \emph{stochastic} quantizers, where in each iteration we randomly choose one of $\tP_i$ to quantize the weights\footnote{This form of stochasticity still leads to \emph{determinisitc networks}, and is therefore conceptually different from \emph{probabilistic} (quantized) networks~\citep{peters2018probabilistic, shayer2018learning}.}. We may also apply different quantizers to different layers of a neural network.
\begin{restatable}{theorem}{proxmap}\label{thm:prox_maps}
Let $\tP_i: \R^d\rightrightarrows\R^d, i = [k]$, be proximal maps. Then, the averaged map 
\begin{align}
\textstyle
\label{eq:pa}
\tP \coloneqq \sum_{i=1}^k \alpha_i \tP_i, \qquad \mbox{ where } \alpha_i \geq 0, ~~\sum_{i=1}^k \alpha_i = 1,
\end{align}
is also a proximal map. Similarly, the product map 
\begin{align}
\textstyle
\label{eq:pp}
\tP \coloneqq \tP_1 \times \tP_2 \times \cdots \times \tP_k, ~~ \wv^* = (\wv_1^*, \ldots, \wv_k^*) \mapsto \big(\tP_1(\wv_1^*), \ldots, \tP_k(\wv_k^*) \big)
\end{align}
is a proximal map (from $\R^{dk}$ to $\R^{dk}$).
\end{restatable}
(The proof of~\Cref{thm:prox_maps} and all other omitted proofs can be found in~\Cref{sec:app-proof}.)

\begin{example}
\label{exm:br}
Let $Q$ be a quantization set (e.g. $Q = \{-1, 0, 1\}$). Clearly, the identity map $\id$ and the projector $\tP_Q$ are proximal maps. Therefore, their convex combination 
$
\tP^\peta = \tfrac{\id + \peta \tP_Q}{1+\peta}, \peta \geq 0, 
$
is also a proximal map, which is exactly the quantizer used in~\citet{yin2018}.
\end{example}

Lastly, we mention that it is intuitively desirable to have $\tP^\peta \to \tP_Q$ when $\peta$ increases indefinitely, so that the quantizer eventually settles on \emph{bona fide} discrete values in $Q$. This is easily achieved by letting minimizers of the underlying function $\rsf$, or the (possibly larger set of) fixed points of $\tP^\peta$, approach $Q$. 
We give one such construction by exploiting \Cref{thm:prox} and \Cref{thm:prox_maps}.

\subsection{General Piecewise Linear Quantizers}
Let $Q_j = \{q_{j, k}\}_{k=1}^{b_j}$ (with $q_{j, 1} \leq \cdots \leq q_{j, b_j}$) be the quantization set for the $j$-th weight group\footnote{In this section, time subscripts are not needed. Instead, we use subscripts to indicate weight groups.} $\rvw_j^* \in \R^{d_j}$. Let $p_{j, k+1} \coloneqq \tfrac{q_{j, k} + q_{j, k+1}}{2}$, $k \in [b_j - 1]$, be the middle points. We introduce two parameters $\rho, \varrho \geq 0$ and define 
\begin{alignat}{3}
&\mbox{horizontal shifts:} \qquad 
&&q_{j,k}^{-} \coloneqq p_{j,k} \vee (q_{j,k} - \rho), ~~ &&q_{j, k}^{+} \coloneqq p_{j,k+1} \wedge (q_{j, k} + \rho), \\
&\mbox{vertical shifts:} \qquad 
&&p_{j,k+1}^{-} \coloneqq q_{j,k} \vee (p_{j,k+1} - \varrho), ~~ &&p_{j, k+1}^{+} \coloneqq q_{j,k+1} \wedge (p_{j, k+1} + \varrho),
\end{alignat}
with $q_{j,1}^{-} = q_{j,1}$ and $q_{j,b_j}^{+} = q_{j, b_j}$. Then, we define $\LP{\rho}$ as the piece-wise linear map (that simply connects the points by straight lines):
\begin{align}
\label{eq:prox-pl}
\LP{\rho}(w^*) \coloneqq \begin{cases}
q_{j,k}, & \mbox{ if } q_{j,k}^- \leq w^* \leq q_{j,k}^+ \\
q_{j,k} + (w^*-q_{j,k}^+) \tfrac{p_{j,k+1}^- - q_{j,k}}{p_{j,k+1} - q_{j,k}^+}, & \mbox{ if } q_{j,k}^+ \leq w^* < p_{j,k+1} \\
p_{j,k+1}^+ + (w^*-p_{j,k+1}) \tfrac{q_{j,k+1} - p_{j,k+1}^+}{q_{j,k+1}^- - p_{j,k+1}}, & \mbox{ if } p_{j,k+1} < w^* \leq q_{j,k+1}^-, 
\end{cases}
\end{align}
for all $w^* \in \rvw_j^*$. At the middle points, $\LP{\rho}(p_{j,k+1})$ may take any value within the two limits. Following the commonly used weight clipping in BC~\citep{courbariaux2015}, we may set $\LP{\rho}(w^*) = q_{j,1}$ for $w^* < q_{j,1}$ and $\LP{\rho}(w^*) = q_{j, b_j}$ for $w^* > q_{j,b_j}$, however, other choices may also work well. The proximal map $\LP{\rho}$ of an example ternary component is visualized in~\Cref{fig:prox_op_evo} for different choices of $\rho$ and $\varrho$.

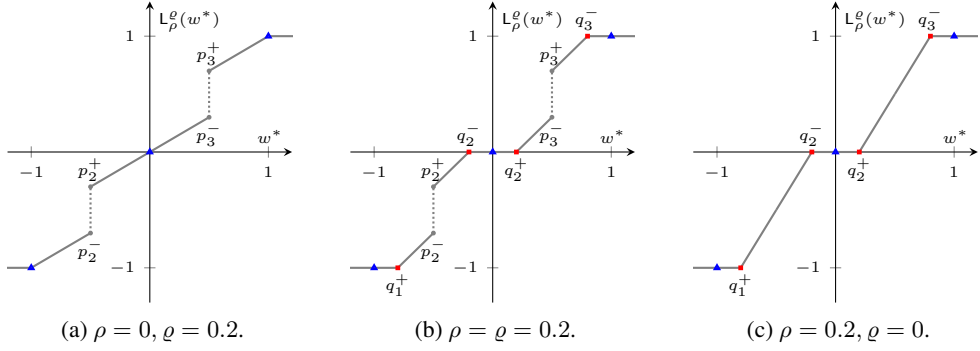
\begin{figure}
    \centering
    \begin{subfigure}[b]{0.32\textwidth}
        \centering
        \begin{tikzpicture}
    [declare function={
        func1(\x) = -1.0;
        func2(\x) = (\x <= -0.5 ) * (0.6*\x - 0.4);
        func3(\x) = and(\x >= -0.5, \x <= 0.5) * (0.6 * \x);
        func4(\x) = (\x >= 0.5) * (0.6 * \x + 0.4);
        func5(\x) = 1.0;}]
       
    \begin{axis}
        [width=1.2\linewidth,
         height=1.25\linewidth,
         axis x line=middle, 
         axis y line=middle,
         ymin=-1.3, ymax=1.3, ytick={-1,...,1}, ylabel=$\LP{\rho}(w^*)$,
         xmin=-1.2, xmax=1.2, xtick={-1,...,1}, xlabel=$w^*$,
         samples=101, every axis plot/.append style={thick},
         soldot/.style={color=gray,only marks,mark=*, mark size=0.5pt},
         holdot/.style={color=blue,only marks,mark=triangle, mark size=1pt},
         voldot/.style={color=red,only marks,mark=square, mark size=0.5pt},
         label style={font=\tiny},
         tick label style={font=\tiny}]
           
        \addplot[domain=-1.7:-1.0, gray] {func1(x)}; 
        \addplot[domain=-1.0:-0.5, gray] {func2(x)};
        \addplot[domain=-0.5:0.5, gray] {func3(x)};
        \addplot[domain=0.5:1.0, gray] {func4(x)};
        \addplot[domain=1.0:1.7, gray] {func5(x)};
        
        \addplot[soldot]coordinates{(-0.5,-0.7)(-0.5, -0.3)(0.5, 0.3)(0.5, 0.7)};
        \addplot[holdot]coordinates{(-1,-1)(0, 0)(1, 1)};
        
        \addplot[gray, densely dotted] coordinates {(-0.5,-0.7)(-0.5,-0.3)};
        \addplot[gray, densely dotted] coordinates {(0.5,0.3)(0.5,0.7)};
        
        \node (p2minus) at (axis cs:-0.5, -0.85) {\tiny $p_2^-$};
        \node (p2plus) at (axis cs:-0.5, -0.15) {\tiny $p_2^+$};
        \node (p3minus) at (axis cs:0.5, 0.15) {\tiny $p_3^-$};
        \node (p3plus) at (axis cs:0.5, 0.85) {\tiny $p_3^+$};
    \end{axis}
\end{tikzpicture} 
        \caption{$\rho=0, \varrho = 0.2$.}
    \end{subfigure}
    \begin{subfigure}[b]{0.32\textwidth}
        \centering
        \begin{tikzpicture}
    [declare function={
        func1(\x) = -1.0;
        func2(\x) = (\x < -1.2 ) * (\x + 0.2) + and(\x >= -1.2, \x <= -0.8) * -1 + and(\x > -0.8, \x <= -0.5) * (\x - 0.2);
        func3(\x) = and(\x >= -0.5, \x < -0.2) * (\x + 0.2) + and(\x >= -0.2, \x <=0.2) * 0 + and(\x > 0.2, \x <= 0.5) * (\x -0.2);
        func4(\x) = and(\x >= 0.5, \x < 0.8) *(\x + 0.2) + and(\x >= 0.8, \x <= 1.2) * 1 + (\x > 1.2) * (\x - 0.2);
        func5(\x) = 1.0;}]
       
    \begin{axis}
        [width=1.2\linewidth,
         height=1.25\linewidth,
         axis x line=middle, 
         axis y line=middle,
         ymin=-1.3, ymax=1.3, ytick={-1,...,1}, ylabel=$\LP{\rho}(w^*)$,
         xmin=-1.2, xmax=1.2, xtick={-1,...,1}, xlabel=$w^*$,
         samples=101, every axis plot/.append style={thick},
         soldot/.style={color=gray,only marks,mark=*, mark size=0.5pt},
         holdot/.style={color=blue,only marks,mark=triangle, mark size=1pt},
         voldot/.style={color=red,only marks,mark=square, mark size=0.5pt},
         label style={font=\tiny},
         tick label style={font=\tiny}]
            
        \addplot[domain=-1.7:-1.0, gray] {func1(x)};
        \addplot[domain=-1.0:-0.5, gray] {func2(x)};
        \addplot[domain=-0.5:0.5, gray] {func3(x)};
        \addplot[domain=0.5:1.0, gray] {func4(x)};
        \addplot[domain=1.0:1.7, gray] {func5(x)};
        
        \addplot[soldot] coordinates{(-0.5,-0.7)(-0.5, -0.3)(0.5, 0.3)(0.5, 0.7)};
        \addplot[holdot] coordinates{(-1,-1)(0, 0)(1, 1)};
        \addplot[voldot] coordinates{(-0.8,-1)(-0.2, 0)(0.2, 0)(0.8, 1)};
        
        \addplot[gray, densely dotted] coordinates {(-0.5,-0.7)(-0.5,-0.3)};
        \addplot[gray, densely dotted] coordinates {(0.5,0.3)(0.5,0.7)};
        
        \node (q1plus) at (axis cs:-0.8,-1.15) {\tiny $q_1^+$};
        \node (p2minus) at (axis cs:-0.5, -0.85) {\tiny $p_2^-$};
        \node (p2plus) at (axis cs:-0.5, -0.15) {\tiny $p_2^+$};
        \node (q2minus) at (axis cs:-0.2, 0.15) {\tiny $q_2^-$};
        \node (q2plus) at (axis cs:0.2, -0.15) {\tiny $q_2^+$};
        \node (p3minus) at (axis cs:0.5, 0.15) {\tiny $p_3^-$};
        \node (p3plus) at (axis cs:0.5, 0.85) {\tiny $p_3^+$};
        \node (q3minus) at (axis cs:0.8,1.15) {\tiny $q_3^-$};
    \end{axis}
\end{tikzpicture} 
        \caption{$\rho=\varrho = 0.2$.}
    \end{subfigure}
    \begin{subfigure}[b]{0.32\textwidth}
        \centering
        \begin{tikzpicture}
    [declare function={
        func1(\x) = -1.0;
        func2(\x) = (\x < -1.2 ) * (5/3*(\x + 1.5) - 1.5) + and(\x >= -1.2, \x <= -0.8) * -1 + and(\x > -0.8, \x <= -0.5) * (5/3*(\x +0.5) - 0.5);
        func3(\x) = and(\x >= -0.5, \x < -0.2) * (5/3*(\x +0.5) - 0.5) + and(\x >= -0.2, \x <=0.2) * 0 + and(\x > 0.2, \x <= 0.5) * (5/3*(\x -0.5) + 0.5);
        func4(\x) = and(\x >= 0.5, \x < 0.8) *(5/3*(\x -0.5) + 0.5) + and(\x >= 0.8, \x <= 1.2) * 1 + (\x > 1.2) * (5/3*(\x - 1.5) + 1.5);
        func5(\x) = 1.0;}]
       
    \begin{axis}
        [width=1.2\linewidth,
         height=1.25\linewidth,
         axis x line=middle, 
         axis y line=middle,
         ymin=-1.3, ymax=1.3, ytick={-1,...,1}, ylabel=$\LP{\rho}(w^*)$,
         xmin=-1.2, xmax=1.2, xtick={-1,...,1}, xlabel=$w^*$,
         samples=101, every axis plot/.append style={thick},
         soldot/.style={color=gray,only marks,mark=*, mark size=0.5pt},
         holdot/.style={color=blue,only marks,mark=triangle, mark size=1pt},
         voldot/.style={color=red,only marks,mark=square, mark size=0.5pt},
         label style={font=\tiny},
         tick label style={font=\tiny}]
            
        \addplot[domain=-1.7:-1.0, gray] {func1(x)};
        \addplot[domain=-1.0:-0.5, gray] {func2(x)};
        \addplot[domain=-0.5:0.5, gray] {func3(x)};
        \addplot[domain=0.5:1.0, gray] {func4(x)};
        \addplot[domain=1.0:1.7, gray] {func5(x)};
        
        \addplot[holdot]coordinates{(-1,-1)(0, 0)(1, 1)};
        \addplot[voldot]coordinates{(-0.8,-1)(-0.2, 0)(0.2, 0)(0.8, 1)};
        
        \node (q1plus) at (axis cs:-0.8,-1.15) {\tiny $q_1^+$};
        \node (q2minus) at (axis cs:-0.2, 0.15) {\tiny $q_2^-$};
        \node (q2plus) at (axis cs:0.2, -0.15) {\tiny $q_2^+$};
        \node (q3minus) at (axis cs:0.8,1.15) {\tiny $q_3^-$};
    \end{axis}
\end{tikzpicture} 
        \caption{$\rho=0.2, \varrho = 0$.}
    \end{subfigure}
    \caption{Different instantiations of the proximal map $\LP{\rho}$ in \eqref{eq:prox-pl} for $Q = \{-1, 0, 1\}$.}
    \vspace{-0em}
    \label{fig:prox_op_evo}
\end{figure}

The (horizontal) parameter $\rho$ controls the discretization vicinity within which a continuous weight will be pulled \emph{exactly} into the discrete set $Q_j$, while the (vertical) parameter $\varrho$ controls the slope (i.e. expansiveness) of each piece. It follows at once from \Cref{thm:prox} that $\LP{\rho}$ is indeed a proximal map. In particular, setting $\rho = 0$ (hence continuous weights are only discretized in the limit) and $\varrho = \tfrac{\peta}{2(1+\peta)}$ we recover \Cref{exm:br} (assuming w.l.o.g. that $q_{j,k+1} - q_{j,k} \equiv 1$). On the other hand, setting $\rho = \varrho$ leads to a generalization of the (binary) quantizer in~\citet{bai2018}, which keeps the slope to the constant 1 (while introducing jumps at middle points) and happens to be the proximal map of the distance function to $Q_j$~\citep{bai2018}. Of course, setting $\rho = \varrho = 0$ yields the identity map and allows us to skip quantizing certain weights (as is common practice), while letting $\rho, \varrho \to \infty$ recovers the projector $\tP_{Q_j}$.

Needless to say, we may adapt the quantization set $Q_j$ and the parameters $\varrho$ and $\rho$ for different weight groups, creating a multitude of quantization schemes. By~\Cref{thm:prox_maps}, the overall operator remains a proximal map.
\section{Demystifying BinaryConnect (BC)} \label{sec:method}
\citet{bai2018} observed that the updates of BC are formally the same as the dual averaging (DA) algorithm \citep{nesterov2009,xiao2010}, even though the latter algorithm was originally proposed and analyzed only for convex problems. A lesser known fact is that (regularized) dual averaging itself is a special case of the generalized conditional gradient (GCG) algorithm. In this section, we first present the aforementioned fact, refine the observation of \citet{bai2018}, and set up the stage for generalizing BC.
\vspace{-1mm}
\subsection{Generalized Conditional Gradient is Primal-Dual}
\vspace{-1mm}
We first present the generalized conditional gradient (GCG) algorithm~\citep{BrediesLorenz08,YuZS17} and point out its ability to solve simultaneously the primal and dual problems. 

Let us consider the following ``regularized'' problem:
\begin{align}
\label{eq:prob-p}
\min_{\rvw \in \R^d} ~ f(\rvw)\coloneqq \ell(\rvw) + \rsf(\rvw),
\end{align}
where $\rsf$ is a general (nonconvex) regularizer. Setting $\rsf$ to the indicator function of $Q$, i.e.,
\begin{align}
\label{eq:ind}
\rsf(\rvw) = \iota_{Q}(\rvw) = \begin{cases} 
0, & \mbox{ if } \rvw\in Q \\ 
\infty, & \mbox{ otherwise } 
\end{cases}
,
\end{align} 
reduces \eqref{eq:prob-p} to the original problem \eqref{eq:prob}. As we will see, incorporating an arbitrary $\rsf$ does not add any complication but will allow us to immediately generalize BC. 

Introducing the Fenchel conjugate function $\ell^*$ (resp. $\rsf^*$) of $\ell$ (resp. $\rsf$): 
\begin{align}
\textstyle
\ell^*(\wv^*) \coloneqq \sup_{\wv} ~ \inner{\wv}{\wv^*} - \ell(\wv),
\end{align}
which is always (closed) convex even when $\ell$ itself is nonconvex,
we state the Fenchel--Rockafellar dual problem \citep[][]{RockafellarWets98}:
\begin{align}
\label{eq:prob-d}
\min_{\rvw^* \in \R^d} ~ \ell^*(-\rvw^*) + \rsf^*(\rvw^*),
\end{align}
which, unlike the original problem \eqref{eq:prob-p}, is always a convex problem.

We apply the generalized conditional gradient algorithm~\citep{BrediesLorenz08,YuZS17} to solving the dual problem\footnote{GCG is usually applied to solving the primal problem \eqref{eq:prob-p} directly. Our choice of the dual problem here is to facilitate later comparison with dual averaging and BinaryConnect.} \eqref{eq:prob-d}: Given $\wv_t^*$, we linearize the function $\rsf^*$ and solve
\begin{align}
\zv_{t}^* = \left[\argmin_{\wv^*} ~ \ell^*(-\wv^*) + \inner{\wv^*}{\wv_t} \right]  =  -\nabla \ell^{**}(\wv_t), ~~ \wv_t \coloneqq \nabla\rsf^*(\wv_t^*),
\end{align}
where we used the fact that $(\nabla\ell^*)^{-1} = \nabla \ell^{**}$. 
Then, we take the convex combination 
\begin{align}
\wv_{t+1}^* = (1-\lambda_t)\wv_t^* + \lambda_t \zv_{t}^*, ~~\mbox{ where } \lambda_t \in [0,1].
\end{align}
The following theorem extends 
\citet[Proposition 4.2,][]{Bach13a} and \citet[Theorem 4.6,][]{Yu13} to any $\lambda_t$: 
\begin{restatable}{theorem}{gcg}\label{thm:gcg}
Suppose $\rsf^*$ is $L$-smooth (i.e. $\nabla\rsf^*$ is $L$-Lipschitz continuous), then for any $\wv$:
\begin{align}
\label{eq:gcg}
\sum_{\tau=0}^t \tfrac{\lambda_\tau}{\pi_\tau} [(\ell^{**} + \rsf^{**})(\wv_\tau) - (\ell^{**} + \rsf^{**})(\wv)]\leq 
(1-\lambda_0)\Delta(\wv,\wv_0) + \sum_{\tau=0}^t \tfrac{\lambda_\tau^2}{2\pi_\tau} L\|\wv_\tau^* - \zv_\tau^*\|_2^2,
\end{align}
where $\wv_t \coloneqq \nabla \rsf^*(\wv_t^*)$, $\pi_t \coloneqq \prod_{\tau=1}^t (1-\lambda_\tau)$, and $\pi_0 \coloneqq 1$. $\Delta(\wv,\wv_t) \coloneqq \rsf^{**}(\wv) - \rsf^{**}(\wv_t) - \inner{\wv-\wv_t}{\wv_t^*}$ is the Bregman divergence induced by the convex function $\rsf^{**}$.
\end{restatable}
While the convergence of $\wv_t^*$ to the minimum of \eqref{eq:prob-d} is well-known (see e.g. \citet{YuZS17}), the above result also implies that a properly averaged iterate $\bar \rvw_t$ also converges to the minimum of the dual problem of \eqref{eq:prob-d}:
\begin{restatable}{cor}{corfourtwo} 
\label{cor:fourtwo}
Let $\bar\wv_t \coloneqq \sum_{\tau=0}^t \Lambda_{t, \tau} \wv_\tau$, where $\Lambda_{t,\tau} \coloneqq \tfrac{\lambda_\tau}{\pi_\tau} / H_t$ and $H_t \coloneqq \sum_{\tau=0}^t \tfrac{\lambda_\tau}{\pi_\tau}$. Then, we have for any $\wv$:
\begin{align}
\label{eq:gcg-a}
(\ell^{**} + \rsf^{**})(\bar\wv_t) - (\ell^{**} + \rsf^{**})(\wv) \leq \frac{(1-\lambda_0)\Delta(\wv,\wv_0)}{H_t} + \frac{L}{2}\sum_{\tau=0}^t \lambda_\tau \Lambda_{t,\tau}\|\wv_\tau^* - \zv_\tau^*\|_2^2.
\end{align}
\end{restatable}
Assuming $\{\zv_\tau^*\}$ is bounded (e.g. when $\ell^{**}$ is Lipschitz continuous), the right-hand side of~\eqref{eq:gcg-a} diminishes if $\lambda_t \to 0$ and $\sum_t \lambda_t = \infty$. Setting $\lambda_t = \tfrac{1}{t+1}$ recovers ergodic averaging $\bar \rvw_t = \tfrac{1}{t+1} \sum_{\tau=0}^t \rvw_\tau$ for which the right-hand side of~\eqref{eq:gcg-a} diminishes at the rate\footnote{The log factor can be removed by setting, for example, $\lambda_t = \tfrac{2}{2+t}$ instead.} $O(\log t / t)$; see~\Cref{app:method} for details.

Thus, GCG solves problem \eqref{eq:prob-d} and its dual simultaneously.
\vspace{-1mm}
\subsection{BC \texorpdfstring{$\subseteq$}{is a subset of} DA \texorpdfstring{$\subseteq$}{is a subset of} GCG} \label{sec:bc}
\vspace{-1mm}
We are now in a position to reveal the relationships among BinaryConnect (BC), (regularized) dual averaging (DA) and the generalized conditional gradient (GCG). 
Since \Cref{thm:gcg} requires $\rsf^*$ to be $L$-smooth, in the event that it is not we resort to a smooth approximation known as the Moreau envelope~\citep{moreau1965proximite}:
\begin{align}
\env_{\rsf^*}^{\peta}(\wv^*) = \min_{\zv^*} ~ \tfrac{1}{2\peta} \|\wv^* - \zv^*\|_2^2 + \rsf^*(\zv^*),
\end{align}
where the minimizer is (by definition) exactly $\tP_{\rsf^*}^{\peta}(\rvw^*)$. 
It is well-known that $\env_{\rsf^*}^\peta$ is $(1/\peta)$-smooth and $(\env_{\rsf^*}^\peta)^* = \rsf^{**}+\tfrac{\peta}{2}\|\cdot\|_2^2$~\citep{moreau1965proximite}. 
We then apply GCG to the approximate dual problem:
\begin{align}
\label{eq:prob-apx}
\min_{\rvw^* \in \R^d} ~ \ell^*(-\rvw^*) + \env_{\rsf^*}^{\peta}(\rvw^*),
\end{align}
whose own Fenchel--Rockfellar dual is:
\begin{align}
\left[\min_{\rvw \in \R^d} ~ \ell^{**}(\rvw) + (\env_{\rsf^*}^\peta)^*(\rvw) \right] = \min_{\rvw \in \R^d} ~ \ell^{**}(\rvw) + \rsf^{**}(\rvw) + \tfrac{\peta}{2}\|\rvw\|_2^2.
\end{align}

The updates of GCG applied to the approximate problem \eqref{eq:prob-apx} are thus:
\begin{align}
\wv_t &\coloneqq \nabla \env_{\rsf^*}^{\peta}(\wv_t^*) = \tP_{\rsf^{**}}^{1/\peta}(\wv_t^*/\peta) \quad (\text{see~\Cref{prop:gcg-dual} for derivation}) \label{eq:gcg-update-moreau}\\
\wv_{t+1}^* &= (1-\lambda_t) \wv_{t}^* + \lambda_t \zv_t^*, ~~\mbox{ where }~~ \zv_t^* = - \nabla \ell^{**}(\wv_t),
\end{align}
which is exactly the updates of (regularized) dual averaging~\citep{xiao2010} for convex problems where\footnote{That is, if we set $\lambda_t = 1/t$ and allow for time-dependent $\mu_t=t$; see~\citet[Algorithm 1,][]{xiao2010}.} $\rsf^{**} = \rsf$ and $\ell^{**} = \ell$. \citet{nesterov2009} motivated dual averaging by the natural desire of non-decreasing step sizes, whereas conventional subgradient algorithms ``counter-intuitively'' assign smaller step sizes to more recent iterates instead. Based on our explanation, we conclude this is possible because dual averaging solves an (approximate) smoothened dual problem, hence we can afford to use a constant (rather than a diminishing/decreasing) step size.

Defining $\pi_t \coloneqq \prod_{s=1}^t (1-\lambda_s)$ for $t \geq 1$, $\pi_0 \coloneqq 1$, $\pi_{-1} \coloneqq (1-\lambda_0)^{-1}$, and setting $\ws_{t}^* = \wv_t^*/\pi_{t-1}$, we have:
\begin{align}
\wv_t = \tP_{\rsf^{**}}^{1/\peta} (\pi_{t-1}\ws_t^*/\peta), \qquad 
\ws_{t+1}^* = \ws_t^* - \tfrac{\lambda_t}{\pi_t} \nabla \ell^{**}(\wv_t).
\end{align}
Let us now reparameterize 
\begin{align}
\textstyle
\eta_t \coloneqq \frac{\lambda_t}{\pi_t} \implies \lambda_t = \frac{\eta_t}{1+\sum_{\tau=1}^t\eta_\tau} \mbox{ and } \frac{1}{\pi_{t}} = 1+\sum_{\tau=1}^t \eta_\tau,
\end{align}
for $t \geq 1$. If we also allow $\peta = \peta_t$ to change adaptively from iteration to iteration (as in dual averaging), in particular, if $\peta_t = \pi_{t-1}$, we obtain the familiar update:
\begin{align} \label{eq:smoothened-gcg-update}
\wv_{t} = \tP_{\rsf^{**}}^{1/\pi_{t-1}} (\ws_{t}^*), \qquad 
\ws_{t+1}^* = \ws_t^* - \eta_t \nabla \ell^{**}(\wv_t)
. 
\end{align}
For nonconvex problems, we may replace $\rsf^{**}$ and $\ell^{**}$ with their nonconvex counterparts $\rsf$ and $\ell$, respectively. 
We remark that $\rsf^{**}$ (resp. $\ell^{**}$) is the largest convex function that is (pointwise) majorized by $\rsf$ (resp. $\ell$). 
In particular, with $\rsf = \iota_Q$ and $\tP_{\rsf}^{1/\peta} = \tP_{Q}$ (for any $\peta$) we  recover the BinaryConnect update \eqref{eq:BC}. 
While the parameter $\mu$ plays no role when $\rsf$ is an indicator function, we emphasize that for general $\rsf$ we should use the quantizer $\tP_{\rsf}^{1/\pi_{t-1}}$, where importantly $1/\pi_{t-1} \to \infty$ hence the quantizer converges to minimizers of $\rsf$ asymptotically.  Neglecting this crucial detail may lead to suboptimality as is demonstrated in the following example: 
\begin{example}
\label{exm:bc}
\citet{bai2018} constructed the following intriguing example: 
\begin{align}
\ell(w) = \tfrac{1}{2} w^2, ~~ Q = \{\pm 1\}, ~~ \tP_{\rsf}^{1/\peta}(w) = \sign(w)\tfrac{\epsilon |w| + \tfrac{1}{\peta}}{\epsilon+\tfrac{1}{\peta}} \mbox{ for $|w| \leq 1$ },
\end{align}
and they showed non-convergence of the algorithm
$
w \gets w - \eta \nabla \ell (\tP_{\rsf}^{1/\peta} (w)),
$
where $\peta$ is a fixed constant.
If we use $\tP^{1/\pi_{t-1}}_{\rsf}$ with some diverging $1/\pi_{t-1}$ instead, the resulting BinaryConnect, with diminishing $\eta_t$ or ergodic averaging, would actually converge to 0 (since $\tP^{1/\pi_{t-1}}_{\rsf} \to \sign$).
\end{example}

\section{ProxConnect (PC): A Generalization of BinaryConnect} \label{sec:PC}
We are now ready to generalize BC by combining the results from \Cref{sec:prox} and \Cref{sec:method}.
Replacing the convex envelopes, $\ell^{**}$ and $\rsf^{**}$, with their nonconvex counterparts and replacing deterministic gradients with stochastic gradients (as well as the change-of-variable $\ws_t^* \to \wv_t^*$), we obtain from~\eqref{eq:smoothened-gcg-update} a \emph{family} of algorithms which we term ProxConnect (PC): 
\begin{align} \label{eq:pc}
\textstyle
\wv_t = \tP_{\rsf}^{1/\pi_{t-1}} (\wv_t^*), \quad 
\wv_{t+1}^* = \wv_t^* - \eta_t \widetilde\nabla \ell(\wv_t),
\end{align}
where the quantizer $\tP_{\rsf}^{1/\pi_{t-1}}$ may be designed directly by following \Cref{sec:prox}.
We have already seen in \Cref{sec:method} that BC belongs to PC by choosing $\tP_{\rsf}^{1/\pi_{t-1}}  = \tP_Q$ (in which case $\pi_{t-1}$ plays no role).

The analysis in \Cref{sec:method}, initially tailored to convex functions, immediately generalizes to the nonconvex algorithm PC (for nonconvex $\ell$, nonconvex $\rsf$, and stochastic gradients $\widetilde\nabla\ell$):
\begin{restatable}{theorem}{pc}\label{thm:pc}
Fix any $\wv$, the iterates in \eqref{eq:pc} satisfy:
\begin{align}
\label{eq:pcb}
\sum_{\tau=s}^t \eta_\tau [\langle\wv_\tau\!-\!\wv, \widetilde\nabla\ell(\wv_\tau)\rangle\!+\! \rsf(\wv_{\tau}) \!-\! \rsf(\wv) ]\leq
\Delta_{s-1}(\wv) \!-\! \Delta_t(\wv) \!+\! 
\sum_{\tau=s}^t \Delta_{\tau}(\wv_{\tau}),
\end{align}
where 
$\Delta_{\tau}(\wv) 
\coloneqq \rsf_{\tau}(\wv) - \rsf_{\tau}(\wv_{\tau+1})  - \inner{\wv-\wv_{\tau+1}}{\wv_{\tau+1}^*}$ is the Bregman divergence induced by the (possibly nonconvex) function $\rsf_\tau(\wv) \coloneqq \tfrac{1}{\pi_\tau} \rsf(\wv) + \tfrac{1}{2}\|\wv\|_2^2$.
\end{restatable}
The summand on the left-hand side of \eqref{eq:pcb} is related to the duality gap in \citet{YuZS17}, which is  
a natural measure of stationarity for the nonconvex problem \eqref{eq:prob-p}. Indeed, it reduces to the familiar ones when convexity is present:
\begin{restatable}{cor}{corpc}
\label{thm:pcc}
For convex $\ell$ and any $\wv$, the iterates in \eqref{eq:pc} satisfy:
\begin{align}
\label{eq:pcc}
\!\!\min_{\tau=s, \ldots, t} ~ \EE [f(\wv_\tau) \!-\! f(\wv)]\leq
\tfrac{1}{\sum_{\tau=s}^t\eta_\tau} \cdot \EE\big[
\Delta_{s-1}(\wv) \!-\! \Delta_t(\wv) \!+\! 
\sum\nolimits_{\tau=s}^t \Delta_{\tau}(\wv_{\tau}) \big].
\end{align}
If $\rsf$ is also convex, then
\begin{align} \label{eq:pcc2}
\min_{\tau=s, \ldots, t} ~ \EE [f(\wv_\tau) \!-\! f(\wv)]\leq
\tfrac{1}{\sum_{\tau=s}^t\eta_\tau} \cdot \EE\big[
\Delta_{s-1}(\wv) \!+\! 
\sum\nolimits_{\tau=s}^t \tfrac{\eta_\tau^2}{2} \|\widetilde\nabla\ell(\wv_\tau)\|_2^2 \big],
\end{align}
and
\begin{align} \label{eq:pcc3}
\EE\big[f(\bar \wv_t) \!-\! f(\wv)\big]\leq
\tfrac{1}{\sum_{\tau=s}^t\eta_\tau} \cdot \EE\big[
\Delta_{s-1}(\wv) \!+\! 
\sum\nolimits_{\tau=s}^t \tfrac{\eta_\tau^2}{2} \|\widetilde\nabla\ell(\wv_\tau)\|_2^2 \big],
\end{align}
where $\rvw_t = \tfrac{\sum_{\tau=s}^t \eta_\tau \rvw_\tau}{\sum_{\tau=s}^t \eta_\tau}$.
\end{restatable}
The right-hand sides of \eqref{eq:pcc2} and \eqref{eq:pcc3} diminish iff $\eta_t\to 0$ and $\sum_t\eta_t = \infty$ (assuming boundedness of the stochastic gradient). We note some trade-off in choosing the step size $\eta_\tau$: both the numerator and denominator of the right-hand sides of \eqref{eq:pcc2} and \eqref{eq:pcc3} are increasing w.r.t. $\eta_\tau$. 
The same conclusion can be drawn for \eqref{eq:pcc} and \eqref{eq:pcb}, where $\Delta_\tau$ also depends on $\eta_\tau$ (through the accumulated magnitude of $\wv_{\tau+1}^*$).
A detailed analysis may need to take specific properties of $\rsf$ or $\tP$ into account~\citep{yin2018}.

\textbf{ProxQuant vs ProxConnect.} It is worthwhile to point out one important difference between ProxQuant and ProxConnect: 
\citet{bai2018} proved convergence (to some notion of stationarity) of ProxQuant for a fixed quantizer~(see~\citet[Theorem 5.1,][]{bai2018}), i.e., $\tP_{\rsf}^\peta$ for a fixed $\peta$, but their experiments relied on incrementing $\peta$ so that their quantizer approaches the projector $\tP_Q$. This creates some discrepancy between theory and practice. The same comment also applies to \cite{ajanthan2021}. In contrast, ProxConnect is derived from a rigorous theory that automatically justifies a diverging $\peta$. In particular, choosing a constant step size $\eta_\tau \equiv \eta_0$ would lead to $1/\pi_{t-1} \propto t$, matching the current practice that is now justifiable if $\rsf$ is strongly convex; see \Cref{app:disc} for details.

\textbf{Connection to Existing Algorithms.} \label{sec:related_work}
Besides the obvious BC, ProxConnect generalizes many other quantization algorithms. As it turns out, many of these algorithms can also be realized using our proposed proximal quantizer $\LP{\rho}$ from~\Cref{sec:prox}; see~\Cref{tab:prox} for a sample summary.

\begin{table}[ht]
    \centering
    \caption{A sample summary of existing quantization algorithms. The PC column indicates if the method is a special case of our proposed ProxConnect algorithm. The $\LP{\rho}$-column indicates if the method uses a quantizer which is a special case of our general quantizer $\LP{\rho}$ introduced in~\Cref{sec:prox}. If so, the $\rho, \varrho$-column states how $\rho$ and $\varrho$ were chosen (in practice): increasing $\shortarrow{1}$, fixed to $0$, or fixed to $\infty$. Other than ProxQuant-Ternary, all methods can compute their quantizers in a single neural network pass. \textdagger: TrainedTernary methods might use a quantizer different than $\LP{\rho}$ to improve performance.}
    \label{tab:prox}
\begin{tabular}{l c c c c c}
    \toprule
    Method & PC & One-pass &Learnable parameters &$\LP{\rho}$ &$\rho, \varrho$\\
    \midrule
    ProxQuant-Binary-$L_1$~\citep{bai2018} &\xmark &\cmark &\xmark &\cmark &$\shortarrow{1}$, $\shortarrow{1}$\\
    ProxQuant-Ternary~\citep{bai2018}      &\xmark &\xmark &\xmark &\xmark      &-  \\
    \midrule
    BinaryConnect~\citep{courbariaux2015} &\cmark &\cmark      &\xmark &\cmark &$\infty$,$\infty$         \\     
    BinaryRelax~\citep{yin2018}  &\cmark       &\cmark &\xmark &\cmark &$0, \shortarrow{1}$ \\
    BinaryWeight~\citep{rastegari2016}  &\cmark      &\cmark &\cmark &\xmark      &-          \\
    MirrorDescentView~\citep{ajanthan2021}  &\cmark     &\cmark &\xmark &\xmark &-   \\
    TrainedTernary\textsuperscript{\textdagger}~\citep{zhu2016} &\cmark     &\cmark &\cmark &\cmark &$\infty$,$\infty$           \\
    TernaryWeight~\citep{li2016arxiv}  &\cmark    &\cmark &\cmark &\xmark      &-        \\
    \midrule
    ProxConnect (ours)  &-    &\cmark &\cmark &\cmark &$\shortarrow{1}$,$\shortarrow{1}$ \\
    \bottomrule
\end{tabular}
    \vspace{-.5em}
\end{table}

\textbf{reverseProxConnect.} In~\Cref{sec:background}, we discussed the idea of reversing BinaryConnect. As ProxConnect generalizes BC, we also present reverseProxConnect~(rPC) as a generalization of reversing BinaryConnect:
\begin{align}
\textstyle
\wv_t = \tP_{\rsf}^{1/\pi_{t-1}} (\wv_t^*), \quad 
\wv_{t+1}^* = \wv_t - \eta_t \widetilde\nabla \ell(\wv_t^*),
\end{align}
In contrast to reversing BinnaryConnect, rPC is not completely without merit: it evaluates the gradient at the continuous weights $\rvw_t^*$ and hence is able to exploit a richer landscape of the loss. Even when stuck at a fixed discrete weight $\rvw_t$, rPC may still accumulate sizable updates (as long as the step size and the gradient remain sufficiently large) to allow it to eventually jump out of $\rvw_t$: note that the continuous weights $\rvw_t^*$ still get updated. Finally, for constant step size $\eta$ and $\pi$, we note that fixed points of rPC, when existent, satisfy:
\begin{align}
    \rvw^* = \tP_{\rsf}^{1/\pi} (\rvw^*) - \eta \nabla \ell(\rvw^*) \iff \rvw^* &= (\id + \eta \nabla \ell)^{-1} \left(\id + \eta \left[\tfrac{\partial \rsf}{\pi \eta}\right]\right)^{-1} (\rvw^*) \\
    &=: \gB\left(\eta, \nabla \ell, \tfrac{\partial \rsf}{\pi \eta} \right)(\rvw^*),
\end{align}
where for simplicity we assumed deterministic gradients $\nabla \ell$ and $\rsf$ to be convex such that $\tP^{1/\pi}_\rsf = (\id + \partial \rsf / \pi)^{-1}$. The operator $\gB$ is known as the backward-backward update (as opposed to the forward-backward update in ProxQuant), and it is known that when $\eta \to 0$ slowly, backward-backward updates converge to a stationary point~\citep{passty1979ergodic}. Thus, despite of our current limited understanding of rPC, there is some reason (and empirical evidence as shown in~\Cref{sec:experiments}) to believe it might still be interesting.

\textbf{Importance of GCG framework:} Deriving PC from the GCG framework may let us transfer recent advances on GCG~\citep{berrada2018deep, zhang2020quantized, dvurechensky2020self} to neural network quantization. Furthermore, it is the cornerstone in justifying the widely-adopted diverging smoothing parameter.
\section{Experiments} \label{sec:experiments}
\subsection{Classification on CIFAR-10} \label{sec:cifar10}
\vspace{-1mm}
We perform image classification on CIFAR-10~\citep{krizhevsky2009} using ResNet20 and ResNet56~\citep{he2016deep}, comparing BinaryConnect~\citep{courbariaux2015} with ProxQuant~\citep{bai2018} and (reverse)ProxConnect using our proposed proximal operator $\LP{\rho}$. For fair comparison, we set $\rho = \varrho$ in $\LP{\rho}$ as this resembles the quantizer (for binary quantization) used in the original ProxQuant algorithm. Similar to~\citet{bai2018}, we increase the parameter $\rho$ (or equivalently $\varrho$) linearly: $\rho_t = \left(1 + t/B\right) \rho_0$. In contrast to~\citet{bai2018}, however, we increase $\rho$ after every gradient step rather than after every epoch as this is more in line with our analysis. We treat $\rho_0$ as a hyperparameter for which we conduct a small grid search. We consider binary ($Q= \{-1, 1\}$), ternary ($Q= \{-1, 0, 1\}$) and quaternary ($Q = \{-1, -0.3, 0.3, 1\}$) quantization. Details for all CIFAR-10 experiments can be found in~\Cref{app:cifar10}.

\textbf{Fine-Tuning Pretrained Models.} In this experiment, all quantization algorithms are initialized with a pretrained ResNet. The test accuracies of the pretrained ResNet20 and ResNet56 are 92.01 and 93.01, respectively.~\Cref{tab:fine_tuning} shows the final test accuracies for the different models. For ProxQuant and (reverse)ProxConnect we respectively picked the best $\rho_0$ values; results for all $\rho_0$ values can be found in~\Cref{cifar10:add_results}. (Reverse)ProxConnect outperforms the other two methods on all six settings.
\begin{table}[t]
    \centering
    \caption{Fine-tuning pretrained ResNets. Final test accuracy: mean and standard deviation in 3 runs.}
    \label{tab:fine_tuning}    
    \begin{tabular}{c  c  c c c c}
    \toprule
    Model &  Quantization & BC~\citep{courbariaux2015} & PQ~\citep{bai2018} & rPC (ours) & PC (ours)\\
    \midrule
    \multirow{3}{*}{ResNet20} & Binary & \textbf{90.31} (0.00) & 89.94 (0.10) & 89.98 (0.17) & \textbf{90.31} (0.21)\\
    & Ternary & 74.95 (0.16) & 91.46 (0.06) &\textbf{91.47} (0.19) & 91.37 (0.18)\\
    & Quaternary & 91.43 (0.07) & 91.43 (0.21) &91.43 (0.06) & \textbf{91.81} (0.14)\\
    \midrule
    \multirow{3}{*}{ResNet56} & Binary & 92.22 (0.12) & 92.33 (0.06) &92.47 (0.29) & \textbf{92.65} (0.16)\\
    & Ternary & 74.68 (1.4) & 93.07 (0.02) &92.84 (0.11) & \textbf{93.25} (0.12)\\
    & Quaternary & 93.20 (0.06) & 92.82 (0.16) &92.91 (0.26) & \textbf{93.42} (0.12)\\
    \bottomrule
\end{tabular}
\end{table}

\textbf{End-To-End Training.} To save computational costs, it is important that quantization algorithms also perform well when they are not initialized with pretrained full-precision model. We therefore compare the four methods for randomly initialized models; see~\Cref{tab:end_to_end} for the results. ProxConnect outperforms all other methods on all six tasks. Interestingly, ProxQuant and reverseProxConnect perform considerably worse for all six tasks when compared to fine-tuning. The performance drop of BinaryConnect and ProxConnect when compared to fine-tuning is only significant for ternary quantization. We found that ProxQuant and reverseProxConnect can be quite sensitive to the choice of $\rho_0$, whereas ProxConnect is stable in this regard; see~\Cref{cifar10:add_results}.
\begin{table}[t]
    \centering
    \caption{End-to-end training of ResNets. Final test accuracy: mean and standard deviation in 3 runs.}
    \label{tab:end_to_end}    
    \begin{tabular}{c  c  c c c c}
    \toprule
    Model &  Quantization & BC~\cite{courbariaux2015} & PQ~\citep{bai2018} & rPC (ours) & PC (ours) \\
    \midrule
    \multirow{3}{*}{ResNet20} & Binary & 87.51 (0.21) & 81.59 (0.75) &81.82 (0.32) & \textbf{89.92} (0.65) \\
    & Ternary & 27.10 (0.21) & 47.98 (1.30) &47.17 (1.94) & \textbf{84.09} (0.16)\\
    & Quaternary & 89.91 (0.09) & 85.29 (0.09) &85.05 (0.27) & \textbf{90.17} (0.14)\\
    \midrule
    \multirow{3}{*}{ResNet56} & Binary & 89.79 (0.45) & 86.13 (1.71) &86.25 (1.50) & \textbf{91.26} (0.59)\\
    & Ternary & 30.31 (7.79) & 50.54 (3.68)&42.95 (1.57) & \textbf{84.36} (0.75)\\
    & Quaternary & 90.69 (0.57) & 87.81 (1.60) &87.30 (1.02) & \textbf{91.70} (0.14) \\
    \bottomrule
\end{tabular}
\end{table}
\vspace{-1mm}
\subsection{Classification on ImageNet} \label{sec:imagenet}
\vspace{-1mm}
We perform a small study on ImageNet~\citep{deng2009} using ResNet18~\citep{he2016deep}. As can be seen in~\Cref{tab:results_imagenet}, BC performs slightly better for fine-tuning whereas PC performs slightly better for end-to-end training. This is not an exhaustive study, but rather a first indication that PC can yield competitive performance on large scale datasets. For more details on the experiments see~\Cref{app:imagenet}.
\begin{table}[ht]
    \centering
    \caption{Fine-tuning (left) and end-to-end training (right). Final test accuracy: mean and standard deviation over three runs.}
    \label{tab:results_imagenet}
    \begin{tabular}{c c c}
    \toprule
    \multirow{2}{*}{BC~\citep{courbariaux2015}} & \multicolumn{2}{c}{PC (ours)} \\ \cmidrule{2-3}
    & $\rho_0=\num{2e-2}$ & $\rho_0=\num{4e-2}$ \\
    \midrule
    \textbf{65.84} (0.04) & 65.44 (0.13) & 65.70 (0.04) \\
    \bottomrule
\end{tabular}
    \begin{tabular}{c c c}
    \toprule
    \multirow{2}{*}{BC~\citep{courbariaux2015}} & \multicolumn{2}{c}{PC (ours)} \\ \cmidrule{2-3}
    & $\rho_0=\num{2.5e-3}$ & $\rho_0=\num{5e-3}$ \\
    \midrule
    63.79 (0.12) &\textbf{63.89} (0.14) & 63.67 (0.12) \\ 
    \bottomrule
\end{tabular}
    \vspace{-.5em}
\end{table}

\section{Conclusion} \label{sec:conclusion}
Capitalizing on a principled approach for designing quantizers and a surprising connection between BinaryConnect and the generalized conditional gradient (GCG) algorithm, we proposed ProxConnect as a unification and generalization of existing neural network quantization algorithms. Our analysis refines prior convergence guarantees and our experiments confirm the competitiveness of ProxConnect. In future work, we plan to apply ProxConnect to training other models such as transformers. The connection with GCG also opens the possibility for further acceleration.

\section*{Acknowledgement}
We thank the anonymous reviewers for their constructive comments as well as the area chair and the senior area chair for overseeing the review process. YY thanks NSERC for (partial) funding support. 

\clearpage
\bibliographystyle{plainnat}
\bibliography{neurips_2021}


\clearpage
\appendix

\section{Proofs}\label{sec:app-proof}
In this section we include all proofs omitted in the main paper, and supply some additional comments. 

\subsection{Proofs for \texorpdfstring{\Cref{sec:prox}}{Section 3}} \label{app:proof_prox}
\proxmap*
\begin{proof}
Define $\tP(\wv) = \sum_i \alpha_i \tP_i(\wv)$ if each $\tP_i(\wv)$ is single-valued. It is known that any proximal map $\tP_i$ is almost everywhere single-valued \cite{RockafellarWets98}, thus $\tP$ is almost everywhere defined. Then, if necessary we take the closure of the graph of $\tP$ so that it is defined everywhere. (This last step can be omitted if we always take the upper or lower limits at any jump of $\tP_i$.)
With this interpretation of the average in \eqref{eq:pa}, the rest of the first claim then follows from \cite[Proposition 4,][]{yu2015}. 

For the second claim about the product map in \eqref{eq:pp}, let $\tP_i$ be the proximal map of $f_i(\wv_i^*)$. Then, it follows immediately from the definition \eqref{eq:proximal_operator} that the product map $\tP$ is the proximal map of the sum function $f(\wv^*) := \sum_i f_i(\wv_i^*)$, where $\wv^* = (\wv_1^*, \ldots, \wv_k^*)$.
\end{proof}

\subsection{Proofs for \texorpdfstring{\Cref{sec:method}}{Section 4}} \label{app:method}
Let us first record a general result for iterates of the following type:
\begin{align}
\label{eq:iter-gen}
\ws_{t+1}^* = \ws_t^* + \eta_t \zv_t^*,
\end{align}
where $\ws_0^*$ is given and $\zv_t^*$ may be arbitrary.
\begin{lemma}
\label{thm:gcg-gen}
For any $\wv$, any sequence of $\wv_t$, and arbitrary function $g$, the iterates \eqref{eq:iter-gen} satisfy:
\begin{align}
\sum_{\tau=s}^t \eta_\tau [ \inner{\wv_\tau \!-\! \wv}{-\zv^*_\tau} \!+\! g(\wv_\tau) \!-\! g(\wv) ] = 
\delta_{s-1}(\wv) \!-\! \delta_t(\wv) \!+\! \sum_{\tau=s}^t \delta_\tau(\wv_\tau),
\end{align}
where $\delta_\tau(\wv) \coloneqq \tfrac{1}{\pi_\tau}g(\wv) - \tfrac{1}{\pi_\tau}g(\wv_{\tau+1}) - \inner{\wv-\wv_{\tau+1}}{\ws_{\tau+1}^*}$ and $\tfrac{1}{\pi_t} = 1+\sum_{\tau=1}^t \eta_\tau$.
\end{lemma}
\begin{proof}
The proof is simple algebra (and was discovered by abstracting the original, tedious proof of \Cref{thm:gcg}). For any $\rvw$, we verify:
\begin{align}
\sum_{\tau=s}^t \eta_\tau &[ \inner{\wv_\tau \!-\! \rvw}{-\zv^*_\tau} \!+\! g(\wv_\tau) \!-\! g(\rvw) ]
\\
&=  
\sum_{\tau=s}^t \inner{\wv_\tau - \rvw}{\ws_{\tau}^*-\ws_{\tau+1}^*} + \eta_\tau [g(\wv_\tau) - g(\rvw) ]
\\
&=  
\sum_{\tau=s}^t \inner{\wv_\tau - \rvw}{\ws_{\tau}^*-\ws_{\tau+1}^*} + (\tfrac{1}{\pi_\tau} - \tfrac{1}{\pi_{\tau-1}})[g(\wv_\tau) - g(\rvw) ]
\\
&=  
\sum_{\tau=s}^t -\inner{\wv_\tau \!-\! \wv_{\tau+1}}{\ws_{\tau+1}^*} \!+\! \inner{\rvw \!-\! \wv_{\tau+1}}{\ws_{\tau+1}^*} \!-\! \inner{\rvw\!-\!\wv_\tau}{\ws_{\tau}^*} \!+\! (\tfrac{1}{\pi_\tau} \!-\! \tfrac{1}{\pi_{\tau-1}})[g(\wv_\tau) \!-\! g(\rvw) ]
\\
\label{eq:rec-gen}
&=
\sum_{\tau=s}^t \delta_{\tau}(\wv_{\tau})  \!-\! \delta_{\tau}(\rvw) \!+\! \delta_{\tau-1}(\rvw).
\end{align}
Telescoping completes the proof. 
\end{proof}

\gcg*
\begin{proof}
We first expand the recursion
\begin{align} \label{eq:gcg-rec-pre}
\wv_{t+1}^* &= (1-\lambda_t)\wv_t^* + \lambda_t \zv_{t}^* \\
\label{eq:gcg-rec}
&= (1-\lambda_t)(1-\lambda_{t-1}) \wv_{t-1}^* + \lambda_t \zv_t^* + (1-\lambda_t) \lambda_{t-1}\zv_{t-1}^*\\
&= \dots \\
&=(1-\lambda_0)\pi_t \wv_0^* + \sum_{\tau=0}^t \tfrac{\pi_t}{\pi_\tau} \lambda_\tau \zv_\tau^*, \label{eq:gcg-rec-last}
\end{align}
where $\pi_t \coloneqq \prod_{\tau=1}^t (1-\lambda_\tau)$ for $t \geq 1$ while $\pi_0 \coloneqq 1$ and $\lambda_t \in [0,1]$. We deduce (by, for instance, setting $\zv_t^* = \wv_t^* \equiv 1$, for all $t \geq 0$, in \eqref{eq:gcg-rec-last}) that 
\begin{align}
\label{eq:pi-tmp}
\frac{1}{\pi_t} = (1-\lambda_0) + \sum_{\tau=0}^t \frac{\lambda_\tau}{\pi_\tau} = 1 + \sum_{\tau=1}^t \frac{\lambda_\tau}{\pi_\tau}.
\end{align}
Define $\tfrac{1}{\pi_{-1}} \coloneqq 1-\lambda_0$ and $\zv_t^* = -\nabla \ell^{**}(\wv_t)$. We apply the convexity of $\ell^{**}$ to the LHS of \eqref{eq:gcg}:
\begin{align}
\mathrm{LHS} &= \sum_{\tau=0}^t \tfrac{\lambda_\tau}{\pi_\tau} [\ell^{**}(\wv_\tau) \!+\! \rsf^{**}(\wv_\tau) \!-\! \ell^{**}(\wv) \!-\! \rsf^{**}(\wv)] \\
&\leq
\sum_{\tau=0}^t \tfrac{\lambda_\tau}{\pi_\tau} [\inner{\wv_\tau - \wv}{-\zv_\tau^*} + \rsf^{**}(\wv_\tau) - \rsf^{**}(\wv) ].
\end{align}
Next, we identify $\eta_\tau \coloneqq \tfrac{\lambda_\tau}{\pi_\tau}$, $g = \rsf^{**}$, $\ws^*_{\tau+1} \coloneqq \wv^{*}_{\tau+1}/\pi_{\tau} = \ws^*_{\tau} + \eta_\tau \zv_\tau^*$, and $\delta_\tau = \tfrac{1}{\pi_\tau}\Delta$ so that we can apply \Cref{thm:gcg-gen}:
\begin{align}
\mathrm{LHS} &\leq
\sum_{\tau=0}^t \tfrac{\lambda_\tau}{\pi_\tau} [\inner{\wv_\tau - \wv}{-\zv_\tau^*} + \rsf^{**}(\wv_\tau) - \rsf^{**}(\wv) ] \\
&\leq 
(1-\lambda_0)\Delta(\wv, \wv_0) - \tfrac{1}{\pi_t}\Delta(\wv, \wv_{t+1}) + \sum_{\tau=0}^t \tfrac{1}{\pi_{\tau}}\Delta(\wv_\tau, \wv_{\tau+1})
\\
&\leq (1-\lambda_0) \Delta(\wv, \wv_0) + 
\sum_{\tau=0}^t \tfrac{1}{\pi_\tau}\tfrac{L}{2} \|\wv_{\tau+1}^* \!-\! \wv_\tau^*\|_2^2,
\end{align}
where in the last step we applied the nonnegativity of the Bregman divergence $\Delta$ (when induced by a convex function such as $\rsf^{**}$) as well as the following inequality:
\begin{align}
\Delta(\wv_\tau, \wv_{\tau+1}) &= \rsf^{**}( \wv_{\tau}) - \rsf^{**}(\wv_{\tau+1}) - \inner{\wv_{\tau} - \wv_{\tau+1}}{\wv_{\tau+1}^*} \\
&= \rsf^{**}( \wv_{\tau}) - \rsf^{**}(\wv_{\tau+1}) - \inner{\wv_{\tau} - \wv_{\tau+1}}{\nabla \rsf^{**}(\wv_{\tau+1})} \quad (\wv_t = \nabla\rsf^{*}(\wv_t^*) \iff \wv_t^* = \nabla\rsf^{**}(\wv_t)) \\
&= \rsf^{***}(\wv_{\tau+1}^*) - \rsf^{***}(\wv_{\tau}^*) - \inner{\wv_{\tau+1}^* - \wv_{\tau}^*}{\nabla \rsf^{***}(\nabla \rsf^{**} (\wv_{\tau}))} \quad (\text{by duality of Bregman divergence}) \\
&= \rsf^{*}(\wv_{\tau+1}^*) - \rsf^{*}(\wv_{\tau}^*) - \inner{\wv_{\tau+1}^* - \wv_{\tau}^*}{\nabla \rsf^*(\nabla \rsf^{**} (\wv_{\tau}))} \quad (\text{by convexity of} \, \rsf^{***})\\
&= \rsf^{*}(\wv_{\tau+1}^*) - \rsf^{*}(\wv_{\tau}^*) - \inner{\wv_{\tau+1}^* - \wv_{\tau}^*}{\wv_{\tau}} \quad \left(\mbox{since }\nabla \rsf^{**} = \left(\nabla \rsf^* \right)^{-1}\right)\\
&\leq \inner{\wv_{\tau+1}^* - \wv_{\tau}^*}{\nabla \rsf^*(\wv_{\tau}^{*})} + \tfrac{L}{2}\|\wv_{\tau+1}^* - \wv_\tau^*\|_2^2 - \inner{\wv_{\tau+1}^* - \wv_{\tau}^*}{\wv_{\tau}} \quad (\text{by smoothness of} \, \rsf^*) \\
&= \tfrac{L}{2}\|\wv_{\tau+1}^* - \wv_\tau^*\|_2^2 \quad (\mbox{since  }\rvw_\tau \coloneqq \nabla \rsf^* (\wv_\tau^*)).
\end{align}
Applying \eqref{eq:gcg-rec-pre} completes the proof.
\end{proof}
\corfourtwo*
\begin{proof}
By the convexity of $\rsf^{**}$ and $\ell^{**}$ as well as the fact that the sum of two convex functions is convex, we have:
\begin{align}
    (\ell^{**}+\rsf^{**})(\bar \rvw_t) \leq \sum_{\tau=0}^t \Lambda_{t, \tau} (\ell^{**}+\rsf^{**})(\rvw_\tau).
\end{align}
Inserting the above in~\Cref{thm:gcg} (multiplied $1/H_t$) and noting that $\sum_{\tau=0}^t \Lambda_{t, \tau} = 1$, we have
\begin{align}
    (\ell^{**}+\rsf^{**})(\bar \rvw_t) - (\ell^{**}+\rsf^{**})(\rvw) &\leq \sum_{\tau=0}^t \Lambda_{t, \tau} \left[(\ell^{**}+\rsf^{**})(\rvw_\tau) - (\ell^{**}+\rsf^{**})(\rvw) \right] \\
    &\leq \tfrac{(1\!-\!\lambda_0)\Delta(\wv,\wv_0)}{H_t} \!+\! \tfrac{L}{2}\sum_{\tau=0}^t \lambda_\tau \Lambda_{t,\tau}\|\wv_\tau^* \!-\! \zv_\tau^*\|_2^2.
\end{align}
\end{proof}
The following instantiation of \Cref{cor:fourtwo} is notable: setting $\lambda_t = \tfrac{1}{t+1}$ implies $\pi_t = \lambda_t$, $H_t = t+1$, and therefore $\bar \rvw_t$ is simply ergodic averaging, i.e., 
\begin{align}
  \bar\wv_t = \frac{1}{t+1}\sum_{s=0}^t \wv_t. 
\end{align} 
The right-hand side of \eqref{eq:gcg-a} diminishes at the rate of $O(\tfrac{\log t}{t})$ since
\begin{align}
    \sum_{\tau=0}^t \lambda_t \Lambda_{t, \tau} = \frac{1}{t+1} \sum_{\tau=0}^t \frac{1}{\tau + 1}, ~~\mbox{and}~~ \sum_{\tau=0}^t \frac{1}{\tau + 1} = O(\log t).
\end{align}
We remark that the log factor can be removed if we set $\lambda_t = \tfrac{2}{t+2}$ instead, for which we have
\begin{align}
    \pi_t &= \prod_{\tau=1}^t \frac{\tau}{\tau + 2} = \frac{2}{(t+1)(t+2)}, \\
    H_t &= \sum_{\tau=0}^t \frac{\lambda_\tau}{\pi_\tau} = \sum_{\tau=0}^t \frac{2}{\tau + 2} \frac{(\tau + 1)(\tau+2)}{2} = \frac{1}{2} (t+1) (t+2), \\
    \sum_{\tau=0}^t \lambda_t \Lambda_{t, \tau} &= \frac{4}{(t+1)(t+2)} \sum_{\tau=0}^t \frac{\tau + 1}{\tau + 2}, ~~\mbox{and}~~ \sum_{\tau=0}^t \frac{\tau + 1}{\tau + 2} = \gO(t).
\end{align}
\begin{prop} 
\label{prop:gcg-dual}
\begin{align}
\nabla \env_{\rsf^*}^{\peta}(\wv_t^*) = \tP_{\rsf^{**}}^{1/\peta}(\wv_t^*/\peta)
\end{align}
\end{prop}
\begin{proof}
By the envelope theorem, we have $\nabla \env_{\rsf^*}^{\peta}(\wv_t^*) = \tfrac{\wv_t^* - \tP_{\rsf^*}^{\peta}(\wv_t^*)}{\peta}$. Furthermore, using the Moreau decomposition, we have
\begin{align}
    \rvw_t^* &= \tP_{\mu \rsf^*}^1(\rvw_t^*) + \tP_{\left(\mu \rsf^*\right)^*}^1(\rvw_t^*) \\
    &= \tP_{\rsf^*}^\mu(\rvw_t^*) + \tP_{ \ssf}^\mu(\rvw_t^*),
\end{align}
where $\ssf(\rvw) \coloneqq \rsf^{**}(\rvw/\peta)$. Combining the two results, we have
\begin{align}
    \nabla \env_{\rsf^*}^{\peta}(\wv_t^*) &= \frac{\wv_t^* - \tP_{\rsf^*}^{\peta}(\wv_t^*)}{\peta} \\
    &= \frac{\tP_{\ssf}^{\peta}(\wv_t^*)}{\peta} \\
    &= \tfrac{1}{\peta} \argmin_\rvw \tfrac{1}{2\peta} \|\rvw - \rvw_t^*\|_2^2 + \ssf(\rvw) \\
    &= \argmin_\rvw \tfrac{1}{2\peta} \|\peta \rvw - \rvw_t^*\|_2^2 + \ssf(\peta \rvw) \\
    &= \argmin_\rvw \tfrac{\peta}{2} \|\rvw - \tfrac{\rvw_t^*}{\peta}\|_2^2 + \rsf^{**}(\rvw) \\ 
    &= \tP_{\rsf^{**}}^{1/\peta}(\rvw_t^* / \mu).
\end{align}

\end{proof}
\subsection{Proofs for \texorpdfstring{\Cref{sec:PC}}{Section 5}} \label{app:proof_PC}
In fact, \Cref{thm:pc} is a special case of a more general result:
\begin{restatable}{theorem}{pcg}\label{thm:pcg}
Given $\wv_0^*$, $\eta_t \geq 0$ and $\tfrac{1}{\pi_t} = 1 + \sum_{\tau=1}^t \eta_\tau$, consider the iterates defined as 
\begin{align}
\label{eq:pc-g}
\wv_t = \tP_{\rsf}^{1/\peta_{t}} (\pi_{t-1}\wv_t^*/\peta_{t}), \qquad 
\wv_{t+1}^* = \wv_t^* - \eta_t \widetilde\nabla \ell(\wv_t).
\end{align}
Then, for any $\wv$, 
\begin{align}
\label{eq:pc-gb}
\sum_{\tau=s}^t \eta_\tau [\langle\wv_\tau\!-\!\wv, \widetilde\nabla\ell(\wv_\tau)\rangle\!+\! \rsf(\wv_{\tau}) \!-\! \rsf(\wv) ] &\leq \Delta_{s-1}(\wv) \!-\! \Delta_t(\wv) \!+\!
\sum_{\tau=s}^t \Delta_{\tau}(\wv_{\tau}) + \\
& \qquad + \sum_{\tau=s}^t \tfrac{1}{2} (\tfrac{\mu_{\tau+1}}{\pi_\tau} - \tfrac{\mu_{\tau}}{\pi_{\tau-1}})(\|\wv\|_2^2- \|\wv_\tau\|_2^2),
\end{align}
where 
$\Delta_{\tau}(\wv) 
\coloneqq \rsf_{\tau}(\wv) - \rsf_{\tau}(\wv_{\tau+1})  - \inner{\wv-\wv_{\tau+1}}{\wv_{\tau+1}^*}$
is the Bregman divergence induced by the (possibly nonconvex) function $\rsf_\tau(\wv) \coloneqq \tfrac{1}{\pi_{\tau}}\rsf(\wv) + \tfrac{\mu_{\tau+1}}{2\pi_\tau}\|\wv\|_2^2$.
\end{restatable}
\begin{proof}
We simply apply \Cref{thm:gcg-gen} with $\zv_\tau^* \coloneqq -\widetilde\nabla\ell(\wv_\tau)$, $\ws_\tau^* = \wv_\tau^*$ and $g \coloneqq \rsf$, and note that 
\begin{align}
\delta_\tau(\wv) = \Delta_\tau(\wv) - \tfrac{\mu_{\tau+1}}{2\pi_\tau}(\|\wv\|_2^2 - \|\wv_{\tau+1}\|_2^2).
\end{align}
To see that $\Delta_\tau$ is the Bregman divergence of $\rsf_\tau$, we apply the (sub)differential optimality condition to 
\begin{align}
\wv_{\tau+1} \coloneqq \tP_{\rsf}^{1/\mu_{\tau+1}}(\pi_{\tau}\wv_{\tau+1}^*/\mu_{\tau+1}) = \argmin_{\wv} ~~ \tfrac{1}{2} \|\wv - \pi_\tau \wv_{\tau+1}^*/\mu_{\tau+1}\|_2^2 + \tfrac{1}{\mu_{\tau+1}}\rsf(\wv),
\end{align}
so that 
\begin{align}
\wv_{\tau+1}^* \in \tfrac{\mu_{\tau+1}}{\pi_\tau}\wv_{\tau+1} + \tfrac{1}{\pi_\tau} \nabla\rsf(\wv_{\tau+1}) = \nabla \rsf_\tau(\wv_{\tau+1})
\end{align}
and hence 
\begin{align}
\Delta_\tau(\wv) = \rsf_\tau(\wv) - \rsf_\tau(\wv_{\tau+1}) - \inner{\wv-\wv_{\tau+1}}{\nabla \rsf_\tau(\wv_{\tau+1})}.
\end{align}
Clearly, $\rsf_\tau$ is convex if $\rsf$ is convex, in which case $\Delta_\tau \geq 0$.
\end{proof}
\pc*
\begin{proof}
Simply set $\peta_t = \pi_{t-1}$ in \Cref{thm:pcg} above.
\end{proof}
\corpc*
\begin{proof}
    We first apply the expectation with respect to random sampling to~\eqref{eq:pcb} which reduces the left hand side to
    \begin{align}
        \EE \left[\sum_{\tau=s}^t \eta_\tau [\langle\wv_\tau\!-\!\wv, \widetilde\nabla\ell(\wv_\tau)\rangle\!+\! \rsf(\wv_{\tau}) \!-\! \rsf(\wv) ] \right] &= \EE \left[\sum_{\tau=s}^t \eta_\tau [\langle\wv_\tau\!-\!\wv, \nabla\ell(\wv_\tau)\rangle\!+\! \rsf(\wv_{\tau}) \!-\! \rsf(\wv) ] \right] \\
        &\geq \EE \left[\sum_{\tau=s}^t \eta_\tau [\ell(\wv_{\tau}) \!-\! \ell(\wv) \!+\!\rsf(\wv_{\tau}) \!-\! \rsf(\wv) ] \right] \\
        &= \EE \left[\sum_{\tau=s}^t \eta_\tau [f(\wv_{\tau}) \!-\! f(\wv) ] \right],
    \end{align}
    where we used the convexity of $\ell$. We then obtain~\eqref{eq:pcc} by using
    \begin{align}
        \!\!\min_{\tau=s, \ldots, t} ~ \EE \left[f(\wv_{\tau}) \!-\! f(\wv) \right] \leq \frac{1}{\sum_{\tau=s}^t\eta_\tau} \EE \left[\sum_{\tau=s}^t \eta_\tau [f(\wv_{\tau}) \!-\! f(\wv) ] \right].
    \end{align}
    The right-hand sides of~\eqref{eq:pcc2} and~\eqref{eq:pcc3} are obtained by setting $\mu_\tau = \pi_{\tau-1}$ and upper bounding the Bregman divergence:
    \begin{align}
        \Delta_\tau(\rvw_\tau) &= \rsf_\tau(\rvw_\tau) - \rsf_\tau(\rvw_{\tau+1}) - \inner{\rvw_\tau - \rvw_{\tau + 1}}{\rvw_{\tau+1}^*} \\
        &= \rsf^*_\tau(\rvw_{\tau+1}^*) - \rsf^*_\tau(\rvw_{\tau}^*) - \inner{\rvw_{\tau+1}^* - \rvw_{\tau}^*}{\rvw_{\tau}} \quad (\text{by duality of Bregman divergence}) \\
        &\leq \frac{1}{2} \|\rvw_\tau^* - \rvw_{\tau + 1}^* \|_2^2 \quad (\text{by 1-smoothness of} \, \rsf^*)\\
        &= \frac{\eta_t^2}{2} \|\widetilde \nabla \ell(\rvw_t) \|_2^2,
    \end{align}
    where we have used the well-known fact that the Fenchel conjugate of a $(1/L)$-strongly convex function is $L$-smooth (in our case $L=1$).
    Lastly, the left-hand side of~\eqref{eq:pcc3} is obtained by applying the convexity of $f$:
    \begin{align}
        \left[{\sum_{\tau=s}^t \eta_\tau}\right]f(\bar \rvw_t) \leq {\sum_{\tau=s}^t \eta_\tau f(\rvw_\tau)},~~\mbox{where}~~ \bar \rvw_t = \tfrac{\sum_{\tau=s}^t \eta_\tau \rvw_\tau}{\sum_{\tau=s}^t \eta_\tau}.
    \end{align}
\end{proof}

\subsubsection{Discussion of \texorpdfstring{\Cref{thm:pcg}}{Theorem A.3}}
\label{app:disc}
Recall the Bregman divergence from~\Cref{thm:pcg}:
\begin{align}
    \Delta_\tau(\rvw) = \rsf_\tau(\rvw) - \rsf_{\tau}(\rvw_{\tau+1}) - \inner{\rvw - \rvw_{\tau + 1}}{\rvw_{\tau+1}^*}, \quad \rsf_\tau(\rvw) = \tfrac{1}{\pi_\tau} \rsf(\rvw) + \tfrac{\mu_{\tau+1}}{2 \pi_\tau} \|\rvw\|_2^2.
\end{align}
When $\rsf$ is $\sigma_0$-strongly convex, $\rsf_\tau$ is $\tfrac{\sigma_0+\mu_{\tau+1}}{\pi_\tau}$-strongly convex, and hence 
\begin{align} \label{eq:bounding_bregman}
\Delta_\tau(\wv_\tau) \leq \tfrac{\pi_\tau}{2(\sigma_0+\mu_{\tau+1})} \|\wv_{\tau+1}^* - \wv_\tau^*\|_2^2 = \tfrac{\pi_\tau\eta_\tau^2}{2(\sigma_0+\mu_{\tau+1})} \|\widetilde\nabla\ell(\wv_\tau)\|_2^2,
\end{align}
where we used the duality of the Bregman divergence and the smoothness of $\rsf_\tau^*$.
Dividing both sides of \eqref{eq:pc-gb} by $\sum_{\tau=s}^t \eta_\tau$ we obtain the upper bound:
\begin{align}
\mathrm{UB} \coloneqq \frac{\tfrac{\pi_{s-1}}{\sigma_0+\mu_{s}}\|\wv^*-\wv_s^*\|_2^2+\sum_{\tau=s}^t \big[ \tfrac{\pi_\tau\eta_\tau^2}{\sigma_0+\mu_{\tau+1}} \|\widetilde\nabla\ell(\wv_\tau)\|_2^2 + (\tfrac{\mu_{\tau+1}}{\pi_\tau} - \tfrac{\mu_{\tau}}{\pi_{\tau-1}})(\|\wv\|_2^2- \|\wv_\tau\|_2^2) \big] }{2\sum_{\tau=s}^t \eta_\tau},
\end{align}
where $\wv^* \coloneqq \nabla \rsf_{s-1}(\wv)$ and we have dropped the non-positive term $-\Delta_t(\rvw)$. 
Suppose\footnote{This assumption can be easily satisfied. In fact, we can just set $\mu_{t+1} = \pi_t$ (as in the main paper), which would simplify the discussion quite a bit.} $\tfrac{\mu_{t+1}}{\pi_t}$ is non-decreasing w.r.t. $t$, we can thus drop some non-positive terms to further simplify: 
\begin{align}
\mathrm{UB} \leq 
\frac{\tfrac{\pi_{s-1}}{\sigma_0+\mu_{s}}\|\wv^*-\wv_s^*\|_2^2 + \tfrac{\mu_{t+1}}{\pi_t} \|\wv\|_2^2 + \sum_{\tau=s}^t \big[ \tfrac{\pi_\tau\eta_\tau^2}{\sigma_0+\mu_{\tau+1}} \|\widetilde\nabla\ell(\wv_\tau)\|_2^2 \big] }{2\sum_{\tau=s}^t \eta_\tau},
\end{align}
To minimize the upper bound, we consider two cases:
\begin{itemize}
\item $\sigma_0 = 0$, in which case let us choose $\mu_{\tau+1} = c \sqrt{\eta_\tau \pi_\tau} = c \sqrt{\lambda_\tau}$ (recall that we reparameterized $\eta_\tau = \lambda_\tau/\pi_\tau$ from GCG), where $c$ is an absolute constant. Then, the upper bound reduces to 
\begin{align}
\frac{\tfrac{\pi_{s-1}}{\mu_{s}}\|\wv^*-\wv_s^*\|_2^2 + \tfrac{c \eta_t}{ \sqrt{\lambda_t}} \|\wv\|_2^2 + \sum_{\tau=s}^t \eta_\tau \sqrt{\lambda_\tau}  \|\widetilde\nabla\ell(\wv_\tau)\|_2^2/c }{2\sum_{\tau=s}^t \eta_\tau},
\end{align}
where recall that $\pi_t = \tfrac{1}{1+\sum_{\tau=1}^t \eta_\tau}$. When the gradient $\nabla\ell$ is bounded, we may choose 
\begin{align}
\label{eq:lambda}
\lambda_t = \tfrac{\eta_t}{ 1 + \sum_{\tau=1}^t \eta_\tau} = O(1/t)
\mbox{ and } \tfrac{1}{\sum_{\tau=1}^t \eta_\tau} = O(1/ \sqrt{t})
\end{align}
so that the upper bound diminishes\footnote{Note that by choosing $s\propto t$, the averaged sequence $\tfrac{\sum_{\tau=s}^t \eta_\tau \sqrt{\lambda_\tau}}{\sum_{\tau=s}^t \eta_\tau} \leq \max\limits_{\tau=s, \ldots, t} \sqrt{\lambda_\tau}$ diminishes similarly in order as $\sqrt{\lambda_t}$.} at the rate of $O(1/\sqrt{t})$. This result makes intuitive sense, since we know the optimal step size $\lambda_t$ in GCG is $\Theta(1/t)$. Also, it reveals that the smoothing parameter $\mu_t = O(1/\sqrt{t})$, matching the rate of the upper bound (i.e., progress on the original problem). Interestingly, the choice in \eqref{eq:lambda} can be realized in multiple ways. In fact, $\eta_t = \eta_0 t^{-p}$ for any $p\in [0, \tfrac12]$ suffices, In particular, choosing $p = 0$ leads to the constant step size $\eta_t \equiv \eta_0$. However, note that these choices are in some sense equivalent, since they all lead to $\lambda_t = O(1/t)$ and $\mu_t = O(1/\sqrt{t})$ hence the underlying GCG progresses similarly.

\item $\sigma_0 > 0$, in which case we can further relax the upper bound to: 
\begin{align}
\mathrm{UB} \leq \frac{\tfrac{\pi_{s-1}}{\sigma_0+\mu_{s-1}}\|\wv^*-\wv_s^*\|_2^2 + \tfrac{\mu_{t+1}}{\pi_t}\|\wv\|_2^2 + \sum_{\tau=s}^t \big[ \tfrac{\pi_\tau\eta_\tau^2}{\sigma_0} \|\widetilde\nabla\ell(\wv_\tau)\|_2^2\big] }{2\sum_{\tau=s}^t \eta_\tau},
\end{align}
and now we may choose $\mu_{\tau+1} = c \eta_\tau\pi_\tau = c \lambda_\tau$, which approaches 0 significantly faster than before. With this choice, the upper bound reduces to 
\begin{align}
\frac{\tfrac{\pi_{s-1}}{\sigma_0+\mu_{s-1}}\|\wv^*-\wv_s^*\|_2^2 + c\eta_t\|\wv\|_2^2 + \sum_{\tau=s}^t  \eta_\tau\lambda_\tau \big[ \|\widetilde\nabla\ell(\wv_\tau)\|_2^2/\sigma_0\big] }{2\sum_{\tau=s}^t \eta_\tau}.
\end{align}
Again, we may set $\lambda_t$ as in \eqref{eq:lambda}, and we can choose constant $\eta_t \equiv \eta_0$. The major difference is that we may now decrease the smoothing parameter $\mu_t$ much more aggressively.
\end{itemize}
\section{Implementation and Experiment Details}
\subsection{CIFAR-10} \label{app:cifar10}
\subsubsection{Model and Quantization Details}
Similar to~\citet{bai2018}, we use ResNets for which we quantize all weights. BatchNormalization layers and activations are kept at full precision. The basic ResNet implementation is taken from~\url{https://github.com/akamaster/pytorch_resnet_cifar10/blob/master/resnet.py}.

\subsubsection{Data Augmentation}
We follow the data augmentation strategy from~\citet{bai2018}: padding by four pixels on each side, randomly cropping to 32-by-32 pixels, horizontally flipping with probability one half. Finally, the images are normalized by subtracting $(0.4914, 0.4822, 0.4465)$ and subsequently dividing by $(0.247, 0.243, 0.261)$.
\subsubsection{Pretrained Full Precision Model Setup}
We pretrain two single full precision models for 200 epochs using the standard optimization setup: SGD with 0.9 momentum and $\num{1e-4}$ weight decay. The initial learning rate is $0.1$ which is multiplied by 0.1 at epoch 100 and 150. We use batch size 128. 
\subsubsection{Fine-Tuning Setup}
All methods are initialized with the last checkpoint of the full precision models.

For BinaryConnect, we train with the recommended strategy from~\citet{courbariaux2015}: Adam with learning rate 0.01 and multiplication of the learning rate by $0.1$ at epoch 81 and 122. For ProxQuant we use Adam with fixed learning rate of 0.01 as in~\citet{bai2018}. We did not perform hyperparameter search over the optimization setup for ProxConnect but rather just used the optimization setup from BinaryConnect as the methods are very similar. We use batch size 128.

Similar to~\citet{bai2018}, we perform a hard quantization at epoch 200: all weights are projected to their closet quantization points. As in~\citet{bai2018}, we then train BatchNormalization layers for another 100 epochs.
\subsubsection{End-To-End Setup}
For simplicity and to avoid an expensive hyperparameter search, we followed the above full precision optimization setup for all methods. Similar to the fine-tuning setup, we perform a hard quantization at epoch 200 and keep training BatchNormalization layers for another 100 epochs. We use batch size 128.
\subsubsection{Compute and Resources} \label{app:cifar10_gpu}
We run all CIFAR-10 experiments on our internal cluster with GeForce GTX 1080 Ti. We use one GPU per experiment. The total amount of GPU hours is summarized in~\Cref{tab:compute_cifar10}.
\begin{table}[ht]
    \centering
    \caption{Compute for CIFAR-10 experiments measured in hours per single GeForce GTX 1080 Ti GPU.}
    \label{tab:compute_cifar10}
    \begin{tabular}{c c c c}
        \toprule
        Architecture & Run time & \# of total experiments & Total run time \\
        \midrule
        ResNet20 & 1.5 & 181 & 271.5 \\
        ResNet56 & 3 & 181 & 543 \\
        \bottomrule
    \end{tabular}
\end{table}
\subsubsection{Additional Results} \label{cifar10:add_results}
As mentioned in~\Cref{sec:cifar10}, we perform a small grid search over $\rho_0$. After an initial exploration stage, we found good regions of $\rho_0$ for ProxQuant and ProxConnect. Since ProxQuant and reverseProxConnect are quite similar, in particular in the small $\rho$ regime, we simply used the $\rho_0$ from ProxQuant for reverseProxConnect. See below the results for all $\rho_0$ settings. \\
\\
ProxConnect is reasonably stable with respect to the choice of $\rho_0$ for both fine-tuning and end-to-end training. ProxQuant and reverseProxConnect, on the other hand, are very sensitive to the choice of $\rho_0$ for end-to-end training. ProxConnect reduces to BinaryConnect for large $\rho$, and therefore it is in line with our experiments that ProxConnect should be stable with respect to $\rho_0$, particularly choosing $\rho_0$ too large should not be an issue.
\begin{table}[ht]
    \centering
    \caption{Additional results for fine-tuning ProxQuant.}
    \label{tab:add_fine_prox_quant}    
    \begin{tabular}{c c c c c c}
    \toprule
    Model &Quantization & $\rho_0=\num{5e-7}$ &$\rho_0=\num{1e-6}$ &$\rho_0=\num{2e-6}$ \\
    \midrule
    \multirow{3}{*}{ResNet20} & Binary &89.68 (0.10) &\textbf{89.94} (0.10) &89.34 (0.33) \\
    & Ternary &91.03 (0.07) &\textbf{91.46} (0.05) &91.11 (0.10)\\
    & Quaternary &91.10 (0.06) &91.13 (0.18) &\textbf{91.43} (0.17)\\
    \midrule
    \multirow{3}{*}{ResNet56} & Binary &92.25 (0.08) &\textbf{92.33} (0.06) &92.16 (0.14)\\
    & Ternary &92.52 (0.90) &\textbf{93.07} (0.02) &92.71 (0.16)\\
    & Quaternary &92.49 (0.03) &\textbf{92.82} (0.13) &92.80 (0.15)\\
    \bottomrule
\end{tabular}
\end{table}
\begin{table}[ht]
    \centering
    \caption{Additional results for fine-tuning reverseProxConnect.}
    \label{tab:add_fine_rbc}    
    \begin{tabular}{c c c c c c}
    \toprule
    Model &Quantization & $\rho_0=\num{5e-7}$ &$\rho_0=\num{1e-6}$ &$\rho_0=\num{2e-6}$ \\
    \midrule
    \multirow{3}{*}{ResNet20} & Binary &89.88 (0.23) &\textbf{89.98} (0.17) &89.91 (0.24)  \\
    & Ternary &91.30 (0.05) &91.36 (0.03) &\textbf{91.47} (0.15) \\
    & Quaternary &91.05 (0.12) &\textbf{91.43} (0.05) &91.42 (0.07) \\
    \midrule
    \multirow{3}{*}{ResNet56} & Binary &92.18 (0.06) &92.35 (0.02) &\textbf{92.47} (0.29)  \\
    & Ternary &\textbf{92.84} (0.09) &92.85 (0.12) &92.84 (0.17) \\
    & Quaternary &92.66 (0.20) &92.88 (0.17) &\textbf{92.91} (0.22) \\
    \bottomrule
\end{tabular}
\end{table}
\begin{table}[ht]
    \centering
    \caption{Additional results for fine-tuning ProxConnect.}
    \label{tab:add_fine_prox_connect}    
    \begin{tabular}{c c c c c c}
    \toprule
    Model &Quantization & $\rho_0=\num{5e-3}$ &$\rho_0=\num{1e-2}$ &$\rho_0=\num{2e-2}$ \\
    \midrule
    \multirow{3}{*}{ResNet20} & Binary &89.63 (0.26) &90.29 (0.07) &\textbf{90.31} (0.21)\\
    & Ternary &91.31 (0.07) &\textbf{91.37} (0.18) &91.13 (0.27) \\
    & Quaternary &91.62 (0.21) &91.55 (0.10) &\textbf{91.81} (0.14) \\
    \midrule
    \multirow{3}{*}{ResNet56} & Binary &92.41 (0.13) &92.62 (0.09) &\textbf{92.65} (0.16) \\
    & Ternary &93.17 (0.04) &\textbf{93.25} (0.12) &93.22 (0.06) \\
    & Quaternary &93.41 (0.11) &\textbf{93.42} (0.12) &93.28 (0.06) \\
    \bottomrule
\end{tabular}
\end{table}
\begin{table}[ht]
    \centering
    \caption{Additional results end-to-end training ProxQuant.}
    \label{tab:add_end_prox_quant}    
    \begin{tabular}{c c c c c c}
    \toprule
    Model &  Quantization & $\rho_0=\num{1e-7}$ &$\rho_0=\num{1e-6}$ &$\rho_0=\num{1e-5}$ \\
    \midrule
    \multirow{3}{*}{ResNet20} & Binary &\textbf{81.59} (0.75) &81.49 (0.41) &71.90 (0.75)\\
    & Ternary &28.22 (1.70) &41.08 (2.95) &\textbf{47.98} (1.06)\\
    & Quaternary &84.58 (0.15) &\textbf{85.29} (0.07) &75.08 (0.16)\\
    \midrule
    \multirow{3}{*}{ResNet56} & Binary &\textbf{86.13} (1.71) &80.25 (0.51) &68.31 (2.21)\\
    & Ternary &21.93 (2.43) &41.11 (2.12) &\textbf{50.54} (3.01)\\
    & Quaternary &\textbf{87.81} (1.30) &83.57 (1.70) &72.58 (2.23)\\
    \bottomrule
\end{tabular}
\end{table}
\begin{table}[ht]
    \centering
    \caption{Additional results end-to-end training reverseProxConnect.}
    \label{tab:add_end_rbc}    
    \begin{tabular}{c c c c c c}
    \toprule
    Model &Quantization & $\rho_0=\num{1e-7}$ &$\rho_0=\num{1e-6}$ &$\rho_0=\num{1e-5}$ \\
    \midrule
    \multirow{3}{*}{ResNet20} & Binary &\textbf{81.82} (0.32) &80.84 (0.40) &72.10 (1.01) \\
    & Ternary &26.49 (2.82) &40.78 (0.39) &\textbf{47.17} (1.94) \\
    & Quaternary &\textbf{85.05} (0.27) &84.82 (0.32) &75.61 (0.54) \\
    \midrule
    \multirow{3}{*}{ResNet56}  & Binary &\textbf{86.25} (1.50) & 81.58 (0.92) &67.53 (2.74) \\
    & Ternary &21.51 (1.12) &\textbf{42.95} (1.57) &36.34 (18.68) \\
    & Quaternary &\textbf{87.30} (1.02) &84.72 (1.31) &73.36 (1.36) \\
    \bottomrule
\end{tabular}
\end{table}
\begin{table}[ht]
    \centering
    \caption{Additional results end-to-end training ProxConnect.}
    \label{tab:add_end_prox_connect}    
    \begin{tabular}{c c c c c c}
    \toprule
    Model &  Quantization & $\rho_0=\num{5e-3}$ &$\rho_0=\num{1e-2}$ &$\rho_0=\num{2e-2}$ \\
    \midrule
    \multirow{3}{*}{ResNet20} & Binary &89.72 (0.13) &\textbf{89.92} (0.26) &89.65 (0.15)\\
    & Ternary &\textbf{84.09} (0.16) &83.54 (0.36) &82.84 (0.34)\\
    & Quaternary &\textbf{90.17} (0.14) &90.12 (0.33) &89.91 (0.09)\\
    \midrule
    \multirow{3}{*}{ResNet56} & Binary &\textbf{91.26} (0.59) &90.45 (0.83) &89.29 (0.45)\\
    & Ternary &\textbf{84.36} (0.75) &83.46 (1.14) &82.54 (1.38)\\
    & Quaternary &91.00 (0.50) &90.76 (0.54) &\textbf{91.70} (0.14)\\
    \bottomrule
\end{tabular}
\end{table}
\clearpage
\subsection{ImageNet} \label{app:imagenet}
\subsection{Model and Quantization Details}
We use ResNet18 as our model, for which we quantize all weights except for the first convolutional layer and the last fully-connected layer. Other components such as BatchNormalization layers, activations and biases are kept at full-precision. The ResNet implementation is borrowed from PyTorch's torchvision package: \url{https://github.com/pytorch/vision/blob/882e11db8138236ce375ea0dc8a53fd91f715a90/torchvision/models/resnet.py}.

\subsection{Data Augmentation}
Common data augmentation strategy is followed from \url{https://github.com/pytorch/examples/blob/c002856901eaf9be112feb9b14a9d5c3e779da74/imagenet/main.py#L204-L231}. At training time, the images are randomly resized and cropped to 224-by-224 pixels, followed by a random horizontal flip. At inference time, images are first resized to 256-by-256 pixels, then center cropped to 224-by-224 pixels. Finally, the images are normalized by subtracting $(0.485, 0.456, 0.406)$ and subsequently dividing by $(0.229, 0.224, 0.224)$.
\subsection{Pretrained Full Precision Model Setup}
We use the ResNet18 checkpoint provided by torchvision as our pretrained model: \url{https://download.pytorch.org/models/resnet18-f37072fd.pth}.

\subsection{Fine-Tuning Setup}
All fine-tuned models are initialized from the pretrained model described in the previous section.

As advised by \citet{alizadeh2018empiricalbnn}, we use Adam as the optimizer to fine-tune the models. We use Adam with default parameters, initial learning rate $\num{1e-4}$ and batch size 256. The models are fine-tuned for 50 epochs. The learning rate is divided by 10 at epoch 15 and 30. We perform hard quantization at epoch 45 and train the remaining full-precision layers for the last 5 epochs, mostly to let BatchNormalization layers stabilize.

\subsection{End-To-End Setup}
For the end-to-end setup, the models are trained from scratch mimicking the training setup of the full-precision pretrained model. The quantized ResNet18 is trained for 90 epochs using SGD with a starting learning rate of $0.1$, momentum of $0.9$ and weight decay of $\num{1e-4}$. The learning rate is multiplied by $0.1$ at epoch 30 and 60. We use batch size 256.  We perform hard quantization at epoch 80 and train the remaining full-precision layers for the last 10 epochs.

\subsection{Compute and Resources} \label{app:imagenet_gpu}
We run all experiments on our internal cluster with Tesla V100. We use one GPU per experiment. The total amount of GPU hours is summarized in~\Cref{tab:compute_imagenet}.
\begin{table}[ht]
    \centering
    \caption{Compute for ImageNet experiments measured in hours per single Tesla V100 GPU.}
    \label{tab:compute_imagenet}
    \begin{tabular}{c c c c}
        \toprule
        Approach & Run time & \# of total experiments & Total run time \\
        \midrule
        End-To-End & 27 & 9 & 243 \\
        Fine-Tuning & 12 & 9 & 108 \\
        \bottomrule
    \end{tabular}
\end{table}

\end{document}